\documentclass[sn-basic]{sn-jnl}% Basic Springer Nature Reference Style/Chemistry Reference Style
\usepackage{graphicx}%
\usepackage{multirow}%
\usepackage{amsmath,amssymb,amsfonts}%
\usepackage{amsthm}%
\usepackage{mathrsfs}%
\usepackage[title]{appendix}%
\usepackage{xcolor}%
\usepackage{textcomp}%
\usepackage{manyfoot}%
\usepackage{booktabs}%
\usepackage{algorithm}%
\usepackage{algorithmicx}%
\usepackage{algpseudocode}%
\usepackage{listings}%
\usepackage{cases}

\usepackage{graphicx} %use graph format
\usepackage{subfigure}
\usepackage{epstopdf}
\usepackage{float}
\usepackage{hyperref}

\usepackage{tabularx}

\usepackage{amsthm,amsmath,amssymb}
\usepackage{mathrsfs}
\usepackage{hyperref}

\usepackage{xcolor}

\usepackage{graphicx}
\usepackage{psfrag}  % 添加 psfrag 宏包

\theoremstyle{thmstyleone}%
\newtheorem{theorem}{Theorem}%  meant for continuous numbers
\newtheorem{proposition}[theorem]{Proposition}% 

\theoremstyle{thmstyletwo}%

\theoremstyle{thmstylethree}%
\newtheorem{definition}{Definition}%

\newcommand{\colorbibs}[2][blue]%
{%
\DeclareBibliographyCategory{ColoredBiblist#1}%
\addtocategory{ColoredBiblist#1}{#2}%
\AtEveryBibitem{\ifcategory{ColoredBiblist#1}{\color{#1}\bfseries}{}}
}

\begin{document}

\title[ ]{Learning using granularity statistical invariants for classification}

\author[1]{\fnm{Ting-Ting} \sur{Zhu}}\email{bxzhuting@163.com}

\author[1]{\fnm{Chun-Na} \sur{Li}}\email
{na1013na@163.com}

\author[2]{\fnm{Tian} \sur{Liu}}\email
{wyliutian@163.com}

\author*[1]{\fnm{Yuan-Hai } \sur{Shao}}\email{shaoyuanhai21@163.com}

\affil[1]{\orgdiv{School of Mathematics and Statistics}, \orgname{Hainan University},  \city{Haikou}, \postcode{570228}, \country{P.R.China}}

\affil[2]{\orgdiv{School of Information and Communication Engineering}, \orgname{Hainan University},  \city{Haikou}, \postcode{570228}, \country{P.R.China}}

% \affil*[2]{\orgdiv{School of Management}, \orgname{Hainan University}, \city{Haikou}, \postcode{570228}, \country{P.R.China}}


\abstract{Learning using statistical invariants (LUSI) is a new learning paradigm, which adopts weak convergence mechanism, and can be applied to a wider range of classification problems. However, the omputation cost of invariant matrices in LUSI is high for large-scale datasets during training. To settle this issue, this paper introduces a granularity statistical invariant for LUSI, and develops a new learning paradigm called learning using granularity statistical invariants (LUGSI). LUGSI employs both strong and weak convergence mechanisms, taking a perspective of minimizing expected risk. As far as we know, it is the first time to construct granularity statistical invariants. Compared to LUSI, the introduction of this new statistical invariant brings two advantages. Firstly, it enhances the structural information of the data. Secondly, LUGSI transforms a large invariant matrix into a smaller one by maximizing the distance between classes, achieving feasibility for large-scale datasets classification problems and significantly enhancing the training speed of model operations. Experimental results indicate that LUGSI not only exhibits improved generalization capabilities but also demonstrates faster training speed, particularly for large-scale datasets.}

%%================================%%
%% Sample for structured abstract %%
%%================================%%

\keywords{Learning using statistical invariants, statistical invariants, reproducing kernel hilbert space, predicates, learning using granularity statistical invariants, granularity}



\maketitle

\section{Introduction}\label{sec1}

An invariant \cite{pfister2019invariant,J2020Parameter,chong2021learning,liu2022learning}  is an important concept in mathematics and physics and refers to a quantity or property that remains unchanged under a certain transformations or operations. Invariants have extensive applications in various fields of machine learning. For example, geometric have been applied invariants in computer vision \cite{mumuni2021cnn} and topological invariants have been used in neural networks \cite{zhang2018machine}. Recently, Vapnik and Izmailov \cite{vapnik2019rethinking} proposed a new invariant learning paradigm for solving classification problems named learning using statistical invariants (LUSI). LUSI incorporates the concept of weak convergence, allowing it to be combined with any method to create a new mechanism that implements both strong and weak convergence. For instance, the $V$-matrix support vector machine (VSVM) \cite{vapnik2015v} can be integrated with the LUSI paradigm to achieve a smoother fitting function. It is important to note that the $V$-matrix can effectively capture the intrinsic structure and distribution information of the data. This information is crucial for constructing statistical invariants to ensure stability under different transformations or rotations. By focusing on the relative spatial relationships between data points, the $V$-matrix represents a statistical invariant, ensuring a robust description of the geometric and distributional characteristics of the data. However, LUSI requires the computation and storage of invariant matrices that are related to the size of the dataset, making it impractical for handling large-scale classification problems. Additionally, existing LUSI learning models consider only the mutual positional relationships among sample points in the entire mixed distribution, neglecting the local positional information of samples, which may result in suboptimal classification outcomes for subsequent learning tasks.

For the first problem, there are two main research approaches for addressing large-scale datasets. The first approach involves using a subset of the training set to represent the entire training set. Representative methods in this category include chunking algorithms \cite{2020Function,viji2021comparative,ellappan2021dynamic} and decomposition algorithms \cite{mohammad2023enhanced,zhou2021comparison}, with the latter being improvements upon the chunking algorithms. The subclass
reduced set - support vector machine (SRS-SVM) \cite{dhamecha2019between} also employs a similar approach. This type of approach requires the selection of appropriate data blocks for training, which may lead to the random loss of some local information in the data. The second approach involves decomposing the large-scale matrix into smaller submatrices, aiming to solve the optimization problem efficiently. For example, Cho-Jui Hsieh et al. \cite{hsieh2014divide} proposed a novel divide-and-conquer solver for kernel SVMs and Hui Xue et al. \cite{xue2011structural} proposed a structural regularized support vector machine algorithm, the idea of which is to cluster the data and decompose the kernel support vector machine into smaller subproblems, enabling each subproblem to be solved independently and effectively. All of the above algorithms employ the idea of granularity \cite{xue2011structural}, addressing the computation and storage issues of large-scale kernel matrices in traditional support vector machines. Granularity construction enables the segmentation of extensive datasets into smaller segments, thereby alleviating the computational and storage burdens. This approach empowers the model to handle substantial data volumes more efficiently. 

For the second problem, incorporating the local positional information of the data is crucial for improving the performance of the classifier. To date, several algorithms that emphasize local information \cite{WEN2020549,WU2020106468} have been developed and widely applied. For example, support vector machine - $K$-nearest neighbor (SVM-KNN) \cite{wang2022particle,LI2020128,deshmukh2018study}, structured large margin machine (SLMM) \cite{gupta2021least,2020Large}, and Laplacian support vector machine (LapSVM) \cite{sun2022mlapsvm,qian2022identification,dong2021robust} are algorithms that emphasize local information. SVM-KNN first finds the $K$-nearest neighbor in the training dataset of each test sample using the KNN \cite{dong2021electrical,zhang2021chameleon,ZHANG2020234} algorithm. These nearest neighbors are subsequently used to construct a local subset of the training data. The SLMM utilizes clustering to explore the underlying distribution of the data, while the LapSVM treats each data point as a granularity. Local information can capture the local structure and similarity between samples, thereby improving the performance of the classifier. This approach provides a new perspective for designing classifiers because classifiers should be sensitive to the structure of the data distribution.

%Inspired by the above analysis, in this paper, we aim to investigate a novel learning paradigm using granularity (LUGSI) to cope with the problems exist in LUSI. LUGSI utilizes clustering techniques to granulate the entire dataset, ensuring that each data point in the dataset contributes to the construction of invariants, with the aim of decomposing the large invariant matrix into smaller grain-based invariant matrices. In addition, based on granularity, LUGSI utilizes the vectors that contain structural information to construct corresponding statistical invariants,  allowing us to focus on the underlying structures within each granule and obtain more detailed local structural information. Experimental results demonstrate the superiority of the LUGSI algorithm compared to LUSI and other related methods. The contributions of this work can be summarized as follows:

Building upon the preceding analysis, this paper explores a groundbreaking learning paradigm that leverages granularity, namely learning using granularity statistical invariants (LUGSI), to address the challenges encountered when using LUSI. LUGSI employs clustering techniques to granulate the entire dataset, ensuring that each data point contributes to the construction of invariants. The primary objective is to decompose the large invariant matrix into smaller grain-based invariant matrices. Additionally, grounded in the concept of granularity, LUGSI utilizes vectors containing structural information \cite{9565320,guo2016support} to construct corresponding statistical invariants. This approach enables a focused examination of the underlying structures within each granule, facilitating the extraction of more detailed local structural information. The experimental results substantiate the superior performance of the LUGSI algorithm compared to that of LUSI and other relevant methods. The noteworthy contributions of this work can be summarized as follows:

   \begin{itemize}
    \item[i)]
    %\item[] 
    We innovatively transforms a large invariant matrix into multiple smaller invariant matrices through the construction of granular statistical invariants. This transformation not only alleviates the computational burden but also demonstrates a novel approach to addressing classification challenges in large-scale datasets. Additionally, this approach contributes significantly to improving the training speed of the model, revealing valuable technical advancements.
    %this paper transforms the large invariant matrix into multiple smaller invariant matrices by constructing granular statistical invariants, which reduces the computation burden. This approach not only effectively addresses the classification problem on large-scale datasets but also significantly improves the training speed of the model.
   \end{itemize}
   \begin{itemize}
    \item[ii)]
    This innovative learning paradigm thoroughly investigates the structural information within a dataset by integrating vectors that encapsulate structural details as statistical invariants, leading to the enhanced classification performance of the model. Importantly, these vectors play a crucial role in estimating the local positional information of samples within the dataset, distinguishing this approach from the solely global positional information considered in LUSI.
    % This novel learning paradigm fully explored the structural information of the dataset by incorporating vectors than contain structural information as statistical invariants, thereby improving the model's classification performance. It is noteworthy that these vectors serve the purpose of estimating the local positional information of samples within the dataset, rather than just the global positional information in LUSI.
   \end{itemize}

   \begin{itemize}
    \item[iii)]
    We explore the relationships between the VSVM, least squares support vector machine (LSSVM) \cite{kadkhodazadeh2021novel}, and LUGSI models. Additionally, equivalent conditions for the LUGSI and the other two models are discussed.
  \end{itemize}

  The remainder of this paper is organized as follows. In Section 2, we briefly introduce the direct estimation method of conditional probabilities used in the VSVM and the definition of a statistical invariant. Section 3 presents the proposed estimation model (LUGSI) derived in this paper. In Section 4, we present the experimental results obtained from several datasets. Finally, Section 5 provides concluding remarks based on our findings. Some common symbols used in this paper in Table \ref{tab11}.

\section{Direct methods of estimation of conditional probability function}\label{sec2}

\begin{table}[]
    \caption{Main notations and their definitions.}\label{tab11}%
    \begin{tabular}{@{}ll@{}}
    \toprule
    Notation& Definition \\
        \midrule
    $n$  & Dimension of a training samples point\\
    $\pmb x_i$ & The $i$-th training sample point of dimension $\left( n\times 1 \right) $\\
    $y_i$ & The label of the $i$-th training sample point\\
    $l$ & The number of training samples\\
    $l_k$ & The number of training samples in the $k$-th granule\\
    $\mathcal{X}$ & The input space\\
    $\left\{ f\left( \pmb x \right) \right\}$ & The admissible function set\\
    $R\left( \pmb x \right)$& The expected risk function\\
    $L\left( \pmb x  \right)$ & The loss function\\
    $P\left( \pmb x,y \right)$ & The joint probability distribution function\\
    $p\left( \pmb x\right)$& The probability density function\\
    $K$ & The kernel matrix of dimension $\left( l\times l \right) $\\
    $K_k$ & The kernel matrix of dimension $\left( l_k\times l \right)$ between the $k$-th  granule and the training set  \\
    $W\left( f \right)$ & The regularization term\\
    $V$ & The positive semidefinite matrix of dimension ($l\times l$) \\
    $\pmb v$ & The vector of dimension ($l\times 1$)\\
    $I$ & The identity matrix of dimension $\left( l\times l \right) $\\
    $m$ & The number of granules or the number of clusters in clustering\\
    $\pmb 1_l$ & The $\left( l\times 1 \right) $ dimensional vector consisting entirely of ones\\
    $ L_2$ & The special subspace in Hilbert space consisting of square-integrable functions that are measurable\\
    $ S_k$ & The set of sample points for the $k$-th cluster formed after the training samples clustering\\
    $ Y_k$ & The label vector corresponding to the $k$-th granule with dimension $\left( l_k\times1  \right) $\\
        \botrule
    \end{tabular}
\end{table}
Recently, Vapnik and Izmailov \cite{vapnik2015v} investigated the binary classification problem by estimating cumulative distribution functions and conditional probabilities. In this approach, different from traditional methods, they directly solved integral equations based on the definitions of conditional probability functions and density functions. 

We now introduce the method for estimating the conditional probability $P(y = 1|\pmb x)$ that is proposed in \cite{vapnik2015v}. Assuming any input vector $\pmb x\in \mathcal{X},$ it is generated from some unknown distribution $P(\pmb x)$. Suppose that some object $\mathcal{O}$ transforms any input vector $\pmb x$ into $y\in \left\{ 0,1 \right\}$. Assuming the transformation between the vector $\pmb x$ and label $y$ is governed by an unknown conditional probability function $P\left( y|\pmb x \right)$. 
%Under this assumption, the classification learning task is to find a function $f(\pmb x)$ in a set of given objective functions $\left\{ f\left( \pmb x \right) \right\}$: $ \mathcal{X}\rightarrow \left\{ 0,1 \right\}$, such that it minimizes the expected risk $$
%R\left( f \right) =\int{L\left( y-f\left( \pmb x \right) \right) dP\left( \pmb x,y \right)}
%$$ defined by some nonnegative loss of $L\left( y-f\left( \pmb x \right) \right)$. 
%Without losing the generality, we suppose $\left\{ f\left( \pmb x \right) \right\}$ is the set of functions that represents the probability of $\pmb x$ belongs to the $y=1$ class. Then we can write  
Then for $f\left( \pmb x \right) =P\left( y=1|\pmb x \right)$, the equality 
\begin{equation}\label{1}
    P\left( y=1|\pmb x \right) p\left( \pmb x \right) =f\left( \pmb x \right) p\left( \pmb x \right) =p\left(  \pmb x ,y=1\right)  
\end{equation}  holds true,
where $
p\left(\pmb x, y=1 \right) 
$ and $
p\left( \pmb x \right) 
$ are  probability density functions. From (\ref{1}), the following equality holds for any function $
G\left( \pmb x-\pmb x' \right) \in L_2
$:
\begin{equation}\label{2}
    \int{G\left( \pmb x-\pmb x' \right)}f\left( \pmb x \right) dP\left( \pmb x \right) =\int{G\left( \pmb x-\pmb x' \right)}dP\left( \pmb x ,y=1\right).   
\end{equation}The conditional probability funtion $
P\left( y=1|\pmb x \right) 
$ is defined by the solution of Fredholm equation (\ref{2}) .%with respect to $
%f\left( \pmb x \right) \in \left\{ f\left( \pmb x \right) \right\} 
%$ when the right-hand side of the equation is known. Therefore, when the probability measures $P\left( \pmb x,y=1 \right) 
%$ and $
%P\left( \pmb x \right) 
%$ are unknown, the estimation of the conditional probability function from the provided data is achieved by solving the corresponding Fredholm integral equation, but with independent and identically distributed $(i.i.d)$ data
%\begin{equation}\label{3}
%\left( \pmb x_1,y_1 \right) ,...,\left( \pmb x_l,y_l \right)
%\end{equation} generated by some unknown $ P(\pmb x,y)$. 

%12.24\subsection{Conditional probability estimation via cumulative distribution function}

When the  joint cumulative distribution function $P \left(\pmb x, y = 1\right)$ and cumulative distribution function $P \left(\pmb x \right) $ are unknown, one needs to solve the equation using the available data $\left\{\left( \pmb x_1,y_1 \right) ,...,\left( \pmb x_l,y_l \right)\right\}$. In specific, the unknown cumulative distribution functions are substituted with their empirical estimates:
%\begin{align}
    $P_l\left( \pmb x \right) =\frac{1}{l}\sum\limits_{i=1}^l{\theta \left( \pmb x-\pmb x_i \right)}\label{eqsystem1}$, $P_l\left( \pmb x,y=1 \right) =\frac{1}{l}\sum\limits_{i=1}^l{y_i\theta \left( \pmb x-\pmb x_i \right)}\label{eqsystem2},$
where the one-dimensional step function is defined as 
$\theta \left( z \right) =\left\{ \begin{array}{l}
	1,\ if\ z\ge 0,\\
	0,\ if\,\,z<0,\\
\end{array} \right.$ and the muti-dimensional step function ($\boldsymbol{\pmb z}=\left( z^1,...,z^n \right)$) is defined as $
\theta \left( \boldsymbol{\pmb z}\right) =\prod\limits_{k=1}^n{\theta \left( z^k \right)}.$

Then from (\ref{2}), we obtain
\begin{equation}\label{5555}
    \frac{1}{l}\sum_{i=1}^l{G\left( \pmb x-\pmb x_i \right)}f\left( \pmb x_i \right) =\frac{1}{l}\sum_{i=1}^l{G\left( \pmb x-\pmb x_i \right)y_i}.     
\end{equation}
In order to estimate the condition probality function, one has  to solve equation (\ref{5555}) in the set of functions  $
\left\{ f\left( \pmb x \right) \right\} 
$. Therefore, it minimizes the distance 
\begin{align}\label{999}
    \rho ^2 &=\int{\left( \sum_{i=1}^l{G\left( \pmb x-\pmb x_i \right) f\left( \pmb x_i \right) -\sum_{j=1}^l{G\left( \pmb x-\pmb x_j \right)y_j}} \right)}^2d\mu \left( \pmb x \right) \nonumber \\
    &=\sum_{i,j=1}^l{\left( f\left( \pmb x_i \right) -y_i \right) \left( f\left( \pmb x_j \right) -y_j \right)}v\left( \pmb x_i,\pmb x_j \right),     
\end{align}
where $\mu \left( \pmb x \right)$ is a probability measure defined on the domain consisting of $\pmb x \in R^n$, and $v\left( \pmb x_i,\pmb x_j \right) =\int{G\left( \pmb x-\pmb x_i \right) G\left( \pmb x-\pmb x_j \right)}d\mu \left( \pmb x \right)$    
is the $\left(i,j\right)$-th element of an ($l\times l$)-dimension positive semidefinite matrix $V$, which is referred to as $V$-matrix \cite{vapnik2019rethinking}.

%12.24\subsection{Basic solution of inference problems in Reproducing Kernel Hilbert space}\label{subsec2.2}

Now, to seek solutions that minimize equation (\ref{999}) within the set of functions $\left\{f\left( \pmb x,\pmb \alpha \right), \pmb \alpha \in \varLambda\right\}$ belonging to the Reproducing Kernel Hilbert Space (RKHS) \cite{bertsimas2022data}, associated with the continuous positive semi-definite kernel function $K (\pmb{x},\pmb{x'})$ defined for $\pmb x,\ \pmb x'\in R^n$ the function to be estimated has the following the representation \cite{vapnik2019rethinking}
\begin{equation}\label{9}
    f\left( \pmb x \right) =\sum_{i=1}^l{\alpha _iK\left( \pmb x_i,\pmb x \right)}+c=\pmb A^T \mathcal{K}\left( \pmb x \right) +c,
\end{equation}
where $\mathcal{K}(\pmb x)=(K(\pmb x_1,\pmb x),...,K(\pmb x_l,\pmb x))^T$, $\pmb A=(\alpha _1,...,\alpha _l)^T$, $\pmb x_i$ are vectors from the training set and $c\in R$ is the bias. The estimation problem of solving Fredholm equation (\ref{2}) is ill-posed, and it is commonly addressed by employing the Tikhonov's regularization \cite{vapnik2020complete} $W\left( f \right) =\lVert f\left( \pmb x \right) \rVert ^2=\sum\limits_{i,j=1}^l{\alpha _i\alpha _jK\left( \pmb x_i,\pmb x_j \right)}.$ Therefore, the solution to find (\ref{9}) needs to minimize

\begin{equation}\label{14}
    R\left( f \right) =\left( F\left( f \right) -Y \right)^T V\left( F\left( f \right) -Y \right) +\gamma W\left( f \right) ,
\end{equation}
where $Y=\left( y_1,y_2,...,y_l \right)^T,$ $F\left( f \right) =\left( f\left( \pmb x_1 \right) ,...,f\left( \pmb x_l \right) \right)^T,$ $V$ is defined as above and $\gamma >0$ is a regularization parameter. Model (\ref{14}) is an unconstraint least squares optimization problem, and its solving is quite routine, and its classification rule is defined as $ r(\pmb x)=\theta(f(\pmb x) - 0.5)$.

Constructing statistical invariants enables the extraction of specific intelligent information from data by preserving statistical properties. The learning using statistical invariants is rooted in the concept of weak convergence in the Hilbert space \cite{vapnik1999nature}. Assuming that we have a set of $m$ functions $\psi_s \left( \pmb x \right)\in L_2 $ that satisfies the following equalities
\begin{equation}\label{21}
    \int{\psi _s\left( \pmb x \right) P\left( y=1|\pmb x \right) dP\left( \pmb x \right) =\int{\psi _s\left( \pmb x \right) dP\left( \pmb x,y=1 \right) =a_s,\ \ s=1,..,m.}}
\end{equation}
Equation (\ref{21}) is referred to as the defining expression for the predicates. It is worth noting that the $a_s$ is the expected value of $\psi _s\left( \pmb x \right) $ with respect to measure $P\left( \pmb x,y=1 \right) $, which is supposed to be known.

By considering a pair $\left( \psi _s\left( \pmb x \right) ,a_s \right) $, one can identify the set $\left\{ f\left( \pmb x \right) \right\} =\left\{ P\left( y=1|\pmb x \right) \right\} $ of conditional probability functions that satisfies the equation (\ref{21}). In real applications, the empirical cumulative distribution functions $P_l\left( \pmb x,y=1 \right)$ and $P_l\left( \pmb x \right) $ are used to approximate $P\left( \pmb x,y=1 \right)$ and $P\left( \pmb x \right)$. Then
\begin{equation}\label{22}
    \frac{1}{l}\sum_{i=1}^l{\psi _s\left( \pmb x_i \right) f(\pmb x_i)}\approx a_s\approx \frac{1}{l}\sum_{i=1}^l{\psi _s\left( \pmb x_i \right)y_i},\ s=1,...,m.
\end{equation}
Equation (\ref{22}) is referred to as the statistical invariant equation. The subset of functions ${f(\pmb x)}$ that satisfies this equation is known as the set of admissible functions, and the expected conditional probability function $f(\pmb x)=P(y=1|\pmb x)$ belongs to this set. To simplify the notations, we define an $l$-dimensional vector: vector $
\varPhi _s=\left( \psi _s\left( \pmb x_1 \right) ,...,\psi _s\left( \pmb x_l \right) \right)^{T} $ of predicate $ \psi _s(\pmb x)$. Then equation (\ref{22}) can be rewritten as 
\begin{equation}\label{23}
    \varPhi _{s}^{T}F\left( f \right) =\varPhi _{s}^{T}Y,\ s=1,...,m.
\end{equation}
$V$-matrices are designed with the goal of maintaining stability across various transformations or rotations, ensuring insensitivity to alterations in the data distribution. This oversight is rectified as the $V$-matrix becomes a key element in the construction of statistical invariants. By specifically considering the relative spatial relationships among data points, the $V$-matrix offers a robust depiction of both the geometric features and distribution characteristics inherent in the data. The subsequent section will delve into a detailed explanation of how the $V$-matrix is employed in the construction of statistical invariants, providing a coherent and logical progression of ideas.

\section{Learning using granularity statistical invariants (LUGSI)
}\label{sec4}

From the VSVM model, it is evident that the scales of $V$-matrix involved is dependent on the number of samples. However, due to the limitations of software memory, they present certain constraints when dealing with large-scale classification problems. In this section, an innovative model is introduced to address this challenge by strategically partitioning the dataset. This partitioning enables the traditional large matrix to be divided into smaller matrices, thus rendering solvability to the large-scale classification problems that were previously unsolvable by the traditional model. However, in cases where certain datasets lack clear intrinsic structure or correlation, partitioning the dataset might not provide additional information to the model, thus failing to yield significant performance improvements.

% Following the research idea of the granularity \cite{liu2020three,xia2019granular,jiang2019accelerator}, the proposed LUGSI is divided into two steps: clustering and learning. LUGSI utilizes the $K$-means \cite{ikotun2023k,borlea2021unified} clustering technique to capture the distribution of data. 

Building upon the research framework centered around granularity \cite{liu2020three,xia2019granular,jiang2019accelerator}, the proposed learning using granularity statistical invariants (LUGSI) is structured into two distinct steps: clustering and learning. LUGSI leverages the $K$-means \cite{ikotun2023k,borlea2021unified} clustering technique to effectively capture the underlying distribution of data. It then incorporates the resulting data distribution information into the learning model. The subsequent sections delve into an exhaustive exploration of each of these steps, providing a comprehensive understanding of their methodologies and implications.

\subsection{Granularization}\label{subsec2}

\begin{definition}
    Given a dataset $T=\left\{ (\pmb x_i,y_i) \right\} _{i=1}^{l}$. Let $S_1,S_2,...,S_m$ be a partition of $T$ according to some relation measure, where the partition characterizes the whole data in the form some sructures such as cluster, and $S_1\cup S_2\cup \cdot \cdot \cdot \cup S_m=T$. Then $S_i$ is called a structural granularity, $i=1,2,...,m$.
\end{definition}

As the structural granularity depends on different assumptions about the actual data structure in real-world problems, the process of granulating the data according to the Definition 1 \cite{xue2011structural} serves not only to facilitate the analysis of data structures but also to convert a substantial volume of data into more compact and manageable units. This transformation contributes to the simplification of the complexity and computational costs associated with data processing. In this paper, we partition the dataset into small data subsets using the $K$-means clustering technique, and we refer to these subsets as \textquotedblleft granules\textquotedblright.

Given a dataset 
$T=\left\{ (\pmb x_i,y_i) \right\} _{i=1}^{l}$, where $\pmb x_i\in R^n
$ and $y_i\in \left\{ 0,1 \right\} $, the goal of the $m$-clustering problem is to partition this dataset into $m$ disjoint subsets or clusters $
S_1,...,S_m
$, while optimizing a clustering criterion. The most commonly employed clustering criterion is the squared Euclidean distance between each data point $
\pmb x_i
$  and the centroid $
\pmb t_k
$ (cluster center) associated with a subset $
S_k
$  of data points $
\pmb x_i
$. This criterion is called clustering error and depends on the cluster centers $
\pmb t_1,...,\pmb t_m
$:
\begin{equation}
    E\left( \pmb t_1,...,\pmb t_m \right) =\sum_{i=1}^l{\sum_{k=1}^m{h\left( \pmb x_i\in S_k \right)}}|\pmb x_i-\pmb t_k|^2,
\end{equation}
where $
h\left( \Omega \right) =1
$ if $\Omega$ is true and 0 otherwise \cite{likas2003global}.

As the $K$-means algorithm requires a predefined number of clusters, which may vary for different datasets, in this article, we consider the number of clusters as a parameter and incorporate it into our experiments.

\subsection{Learning model and solution}\label{subsec333333}

After performing $K$-means clustering on the training set, we obtain $m$ granules denoted as $
S_1,...,S_m
$, with corresponding labels $
Y_1,...,Y_m
$. Let $
l_1,...,l_m
$ represent the number of data points in each granule. 
We construct a corresponding statistical invariant for each granule. Thus, $m$ granules can construct $m$ statistical invariants. Here we use the $\pmb v$ vector \cite{MXZ2023} as a predicate to construct statistical invariant.
% For given predicate $
% \psi _s\left( \pmb x \right)$, the equation 

% \begin{equation}
%     \frac{1}{l}\sum_{i=1}^l{\psi _s\left( \pmb x_i \right) f\left( \pmb x_i \right)}=\frac{1}{l}\sum_{i=1}^l{\psi _s\left( \pmb x_i \right)y_i},\ s=1,...,m
% \end{equation}
% holds for each granule.

In specific, in the space of $\hat{\pmb x}$ (here, $\hat{\pmb x}$ is any sample from the distribution, regardless of whether it is labeled or continuous.), for each training sample $\pmb x_i$, we compute
\begin{equation}\label{30}
    v (\pmb x_i)=\int{G\left( \hat{\pmb x}-\pmb x_i \right)}d\mu \left( \hat{\pmb x} \right), i=1,...,l.
\end{equation}
Each sample is associated with a corresponding $\pmb v$ value, as demonstrated in equation (\ref{30}), and for a total of $l$ samples, we obtain an $l\times 1$ dimensional $\pmb v$ vector :
$
\pmb v=\left(v(\pmb x_1),v(\pmb x_2),...,v(\pmb x_l) \right)^{T}
$. So after clustering, the $\pmb v$ vector can be derived as
\begin{align*}
    &\pmb v_1=\left( v\left( \pmb x_1 \right) ,v\left( \pmb x_2 \right) ,...,v\left( \pmb x_{l_1} \right) \right)^{T}, \,\,\,\,\pmb x_i\in S_1,\\
    &\pmb v_2=\left( v\left( \pmb x_1 \right) ,v\left( \pmb x_2 \right) ,...,v\left( \pmb x_{l_2} \right) \right)^{T}, \ \ \pmb x_i\in S_2,\\               
    &\dots\\
    &\pmb v_m=\left( v\left( \pmb x_1 \right) ,v\left( \pmb x_2 \right) ,...,v\left( \pmb x_{l_m} \right) \right)^{T}, \,\,\,\,\pmb x_i\in S_m.
\end{align*}

Note that the $v(\pmb x_i)$ is the integral over $\hat{\pmb x}$ for the training sample $\pmb x_i$ in the distribution.
Therefore, $\pmb v$ can be obtained through both parameter estimation and non-parameter estimation methods mentioned above. In the equation provided, $\pmb x_i$ represents the training sample with its corresponding label $y_i$, and $\hat{\pmb x}$ can be any value from the unknown distribution $P(\hat{\pmb x})$. From this perspective, it may be more accurate to estimate the relative positions of the estimated training sets by incorporating prior probabilities or samples with unknown labels, rather than solely estimating the training set. Additionally, through this approach, it is possible to establish a relationship between the training set (labeled samples) and the test set (unlabeled samples) \cite{MXZ2023}.

Therefore, $\pmb v_{i}$ can be used to construct the respective statistical invariant for each granule and the corresponding $m$ statistical invariants can be obtained as follows:
\begin{align*}
 \sum_{i=1}^{l_1}{v \left( \pmb x_i \right) f\left( \pmb x_i \right)}&=\sum_{i=1}^{l_1}{v \left( \pmb x_i \right) y_i},\quad \pmb x_i\in S_1,\\
 \sum_{i=1}^{l_2}{v \left( \pmb x_i \right) f\left( \pmb x_i \right)}&=\sum_{i=1}^{l_2}{v \left( \pmb x_i \right) y_i},\quad  \pmb x_i\in S_2,\\ 
                                                              &\dots\\
 \sum_{i=1}^{l_m}{v \left( \pmb x_i \right) f\left( \pmb x_i \right)}&=\sum_{i=1}^{l_m}{v \left( \pmb x_i \right) y_i},\quad  \pmb x_i\in S_m.\\
\end{align*}
In oder to solve the above equations in the set of functions $
\left\{ f\left( \pmb x \right) \right\} 
$, we minimize the distance between each granule: %\newpage
\begin{align*}
    \rho _{k}^{2} &=\left( \sum_{i=1}^{l_k}{v \left( \pmb x_i \right) f\left( \pmb x_i \right)}-\sum_{j=1}^{l_k}{v \left( \pmb x_j \right) y_j} \right) ^2 \\
                  &=\sum_{i,j=1}^{l_k}{f\left( \pmb x_i \right) f\left( \pmb x_j \right)}V_k(i,j)-2\sum_{i,j=1}^{l_k}{f\left( \pmb x_i \right) y_j}V_k(i,j) +\sum_{i,j=1}^{l_k}{y_iy_jV_k(i,j)} \\
                  &=\left( F\left( f_k \right) -Y_k \right)^{T} V_k(i,j)\left( F\left( f_k \right) -Y_k \right),
\end{align*}
where $\pmb x_i\in S_i$, and $V_k=\pmb {v}_k\pmb {v}_k^{T}$.

Then, we get the function:
\begin{align*}
\rho ^2&=\rho _{1}^{2}+\rho _{2}^{2}+...+\rho _{m}^{2}\\
       &=\sum_{k=1}^m{\left( F\left( f_k \right) -Y_k \right) }^{T}V_k\left( F\left( f_k \right) -Y_k \right).       
\end{align*}

To obtain a linear solution, we set 
\begin{equation}\label{31}
    f\left( \pmb x \right) = \pmb w^T\pmb x+b.
\end{equation}
We use a regularization method to estimate the conditional probability function. Therefore, the function to be minimized is: 
\begin{equation}\label{32}
    R_k\left( f \right) =\rho_k ^2+\gamma W\left( f \right).
\end{equation}

By substituting equation (\ref{31}) into equation (\ref{32}) and replacing the variables with $\pmb w$ and $b$, the objective function of each granule is obtained as
\begin{equation}
    R_k\left( \pmb w,b \right) =\left( X_k\pmb w+b\pmb 1_{l_k} \right) ^TV_k\left( X_k\pmb w+b\pmb 1_{l_k} \right) -2\left( X_k\pmb w+b\pmb 1_{l_k} \right) ^TV_kY_k+Y_k^{T}V_kY_k+\gamma \pmb w^T\pmb w,
\end{equation}
where the $\gamma $ represents the parameter for the regularization item.
The final objective function is obtained by accumulation:
\begin{align*}
R\left( \pmb w,b \right) &=\sum_{k=1}^m{R_k\left( \pmb w,b \right)}\\
                    &=\sum_{k=1}^m{\left( X_k\pmb w+b\pmb 1_{l_k} \right) ^TV_k\left( X_k\pmb w+b\pmb 1_{l_k} \right) -2\left( X_k\pmb w+b\pmb 1_{l_k} \right) ^TV_kY_k+Y_k^{T}V_kY_k+\gamma m\pmb w^T\pmb w}.
\end{align*}
The necessary conditions of minimum are
\begin{equation}\label{34}
\left\{ \begin{array}{l}
	\frac{\partial R\left( \pmb w,b \right)}{\partial \pmb w}=\sum\limits_{k=1}^m{X_{k}^{T}V_kX_k\pmb w+bX_{k}^{T}V_k\pmb 1_{l_k}-X_{k}^{T}V_kY_k+\gamma m\pmb w=\pmb 0},\\
	\frac{\partial R\left( \pmb w,b \right)}{\partial b}=\sum\limits_{k=1}^m{\pmb 1_{l_k}^{T}V_kX_k\pmb w+b\pmb 1_k^{T}V_k\pmb 1_{l_k}-\pmb 1_{l_k}^{T}V_kY_k=0}.\\
\end{array} \right. 
\end{equation}
From equation (\ref{34}), we obtain
\begin{equation}\label{35}
    \pmb w=\left[ \sum_{k=1}^m{\left( X_{k}^{T}V_kX_k+\gamma I \right)} \right] ^{-1}\left[ \sum_{k=1}^m{\left( X_{k}^{T}V_kY_k-bX_{k}^{T}V_k\pmb 1_{l_k} \right)} \right].
\end{equation}
We then compute vectors
\begin{equation}\label{36}
    \pmb w_b=\left[ \sum_{k=1}^m{\left( X_{k}^{T}V_kX_k+\gamma I \right)} \right] ^{-1}\left[ \sum_{k=1}^m{X_{k}^{T}V_kY_k} \right] ,\end{equation}
\begin{equation}\label{37}
    \pmb w_c=\left[ \sum_{k=1}^m{\left( X_{k}^{T}V_kX_k+\gamma I \right)} \right] ^{-1}\left[ \sum_{k=1}^m{X_{k}^{T}V_k\pmb 1_{l_k}} \right].
\end{equation}
According to (\ref{35}), the desired vector $\pmb w$ has the form
\begin{equation}\label{38}
    \pmb w=\pmb w_b-b\pmb w_c.\end{equation}
By substituting equation (\ref{38}) into the second equation of (\ref{34}), we obtain:
\begin{equation}
b=\frac{\left[ \sum\limits_{k=1}^m{\pmb 1_{l_k}^{T}V_kY_k-\sum\limits_{k=1}^m{\pmb 1_{l_k}^{T}V_kX_k}\cdot \pmb w_b} \right]}{\left[\sum\limits_{k=1}^m{\pmb 1_{l_k}^{T}V_k\pmb 1_{l_k}}-\sum\limits_{k=1}^m{\pmb 1_{l_k}^{T}V_kX_k}\cdot \pmb w_c\right]}.
\end{equation}
 Putting $b$ in (\ref{38}), we obtain the desired parameter $\pmb w$. The desired function $f(\pmb x)$ has the form (\ref{31}). After the model is solved, we use $\theta(f(\pmb x)-0.5)$ to predict the labels of sample points.

\subsection{Nonlinear LUGSI}\label{subsec222}

For the nonlinear LUGSI, we have $ f\left( \pmb x \right) =\pmb A^T \mathcal{K}\left( \pmb x \right) +c$ and the classification rule is denoted as $\theta(f(\pmb x) - 0.5)$. Substituting $f(\pmb x)$ into equation (\ref{32}), the independent variables of the objective function become $\pmb A$ and $c$. Then the objective function of the $k$-th granule LUGSI in the kernel space can be described as follows:
\begin{equation}
    R_k\left( \pmb A,c \right) =\left( K_k\pmb A+c\pmb 1_{l_k} \right) ^TV_k\left( K_k\pmb A+c\pmb 1_{l_k} \right) -2\left( K_k\pmb A+c\pmb 1_{l_k} \right) ^TV_kY_k+Y_k^{T}V_kY_k+\gamma \pmb A^T\pmb A.
\end{equation}

It is worth mentioning that $K_k$ represents the kernel matrix between the $k$-th granule and the training set. In these notations, the regularization term $W$ can be written as $W=\pmb A^T\pmb A$. By accumulating $R_k\left( \pmb A,c \right) $, we obtain the final objective function
 \begin{align*}
 R\left( \pmb A,c \right) &=R_1\left( \pmb A,c \right) +R_2\left( \pmb A,c \right) +,...,+R_m\left( \pmb A,c \right) \\
                   &=\sum_{k=1}^m{\left( K_k\pmb A+c\pmb 1_{l_k} \right) ^TV_k\left( K_k\pmb A+c\pmb 1_{l_k} \right) -2\left( K_k\pmb A+c\pmb 1_{l_k} \right) ^TV_kY_k+Y_k^{T}V_kY_k+\gamma m\pmb A^T\pmb A}.
 \end{align*}
 The solution process is the same as described in Section \ref{subsec333333} and is summarized briefly here.
The necessary conditions for the minimum are
 \begin{equation}
    \left\{ \begin{array}{l}
        \frac{\partial R\left( \pmb A,c \right)}{\partial \pmb A}=\sum\limits_{k=1}^m{K_{k}^{T}V_kK_k\pmb A+{cK_{k}^{T}V_k\pmb 1_{l_k}-{K_{k}^{T}V_kY_k}+\gamma m\pmb A= \pmb 0}},\\
        \frac{\partial R\left( \pmb A,c \right)}{\partial c}=\sum\limits_{k=1}^m{\pmb 1_{l_k}^{T}V_kK_k\pmb A+c\pmb 1_k^{T}V_k\pmb 1_{l_k}-\pmb 1_{l_k}^{T}V_kY_k=0,}\\
    \end{array} \right. 
 \end{equation}
 where $\pmb A=\pmb A_b-c\pmb A_c$, and
%$$
\begin{equation}\label{1234}
    \left\{ \begin{array}{l}
        \pmb A_b=\left[ \sum\limits_{k=1}^m{\left( K_{k}^{T}V_kK_k+\gamma I \right)} \right] ^{-1}\sum\limits_{k=1}^m{K_{k}^{T}V_kY_k},\\
        \pmb A_c=\left[ \sum\limits_{k=1}^m{(K_{k}^{T}V_kK_k+\gamma I)} \right] ^{-1}\sum\limits_{k=1}^m{K_{k}^{T}V_k\pmb 1_{l_k}},\\
        c=\frac{\left[ \sum\limits_{k=1}^m{\pmb 1_{l_k}^{T}V_kY_k-\sum\limits_{k=1}^m{\pmb 1_{l_k}^{T}V_kK_k\cdot\pmb  A_b}} \right]}{\left[\sum\limits_{k=1}^m{\pmb 1_{l_k}^{T}V_k\pmb 1_{l_k}-\sum\limits_{k=1}^m{\pmb 1_{l_k}^{T}V_kK_k\cdot \pmb A_c}}\right]}.\\
    \end{array} \right. 
\end{equation}
%$$
% \textcolor{red}{Similarly, in this context, the labels of sample points are predicted using the function $\theta(f(\pmb x) - 0.5)$. 

The whole procedure is described
in Algorithm 1.
\begin{center}
    \resizebox{\textwidth}{!}{
    \begin{tabular}{lllll}
        \toprule
        \noindent{\pmb{Algorithm~1} LUGSI algorithm}\\
        \midrule
        \textbf{Input}: Training dataset $T=\{\left( \pmb x_1,y_1 \right) ,...,\left( \pmb x_l,y_l \right)\}$, parameters $C$, $\gamma$, $m$ and $\delta$.\\
        %\midrule
        \textbf{Process}: \\
        
        ~~{1.~}Perform $K$-means clustering on the training set $T$, dividing it into $m$ granules;\\
        
        ~~{2.~}\pmb{for} $i=1,..., m$ :\\
        
        ~~{3. }  Calculate the $V_k$ matrix and the kernel matrix $K_k$  corresponding to each granule;\\
        
        ~~{4. }  Set intermediate variable $
        E_i=K_{k}^{T}V_kK_k+\gamma I
        $, $
        F_i=K_{k}^{T}V_kY_k
        $, $
        H_i=K_{k}^{T}V_k\pmb{1}_{l_k}
        $, \\
        ~~{~ ~}  $
        J_i=\pmb{1}_{l_k}^{T}V_kY_k
        $, $
        L_i=1_{l_k}^{T}V_kK_k
        $, $
        R_i=\pmb{1}_{l_k}^{T}V_k\pmb{1}_{l_k}
        $ and perform cumulative summation;\\
        
        ~~{5. }\pmb{end}\\

        ~~{6. } Substitute the results obtained from steps 3 to step 5 into the system of equations (\ref{1234}) \\ 
        ~~{~ ~} to obtain $A_b$, $A_c$ and c.\\
        
        \textbf{Output}: $A_b$, $A_c$ and $c$.\\

        \textbf{Predict}:
         Assign $\pmb{x}$ to class $y=\theta(f(\pmb{x})-0.5)$.\\
        \bottomrule
    \end{tabular}
    }
    \label{Algorithm1}
\end{center}

Additionally, this paper introduces another enhanced method for handling nonlinear classification called discrete cosine transform learning using granularity statistical invariants (DCTLUGSI). Unlike LUGSI, DCTLUGSI incorporates the CRO kernel and DCT algorithm \cite{kafai2017croification}. This approach efficiently transforms vectors in the input space into sparse high-dimensional vectors in the feature space, striking a balance between accuracy and efficiency in nonlinear spaces. We denote $D(\pmb x_i)$ as the DCT map to compute the CRO kernel. As in the preceding derivation, where we set the slopes in the form of $\pmb w=\pmb w_b-b\pmb w_c$ or $\pmb A=\pmb A_b-c\pmb A_c$, our task of finding the slopes in the model transforms into finding $\pmb w_b$, $\pmb w_c$ or $\pmb A_b$, $\pmb A_c$. The classification rule for data points remains as 
$\theta(f(\pmb x) - 0.5)$. Therefore, we directly provide the analytical solutions for $\pmb A_b$, $\pmb A_c$ and the bias $c$ as shown below:
\begin{align*}
\left\{ \begin{array}{l}
	\pmb A_b=\left[ \sum\limits_{k=1}^{m}{\left( D\left( X_k \right) ^TV_kD\left( X_k \right) +\gamma I \right)} \right] ^{-1}\sum\limits_{k=1}^m{D\left( X_k \right) ^TV_kY_k},\\
	\pmb A_c=\left[ \sum\limits_{k=1}^m{\left( D\left( X_k \right) ^TV_kD\left( X_k \right) +\gamma I \right)} \right] ^{-1}\sum\limits_{k=1}^m{D\left( X_k \right) ^TV_k\pmb 1_{l_k}},\\
    c=\frac{\left[ \sum\limits_{k=1}^m{\pmb 1_{l_k}^{T}V_kY_k}-\sum\limits_{k=1}^m{\pmb 1_{l_k}^{T}V_kD\left( X_k \right)}\cdot \pmb A_b \right]}{\left[\sum\limits_{k=1}^m{\pmb 1_{l_k}^{T}V_k\pmb 1_{l_k}}-\sum\limits_{k=1}^m{\pmb 1_{l_k}^{T}V_kD\left( X_k \right)}\cdot \pmb A_c\right]}.\\
\end{array} \right. 
\end{align*}

DCTLUGSI combines the characteristics of the CRO kernel, which typically exhibits higher computational efficiency compared to the Gaussian kernel. This enables the reduction of computational costs and accelerates training time when dealing with large-scale datasets.

\subsection{Relationship with LUGSI, VSVM and LSSVM}\label{subsec2}

% In this section, we discuss the relationship between LUGSI, VSVM, and LSSVM. 

In this section, we will describe the equivalence of these three algorithms under certain conditions. These analyses contribute to a more profound understanding of the distinctive characteristics inherent in LUGSI.
\begin{proposition}
    When each point is treated as an individual class, the LUGSI model degenerates into the LSSVM model. Additionally, if we consider the $V$-matrix as an identity matrix, the VSVM model degenerates into LSSVM.
\end{proposition}
\begin{proof}    
    Evidently, neglecting the relative positional relationships between sample points within each granule occurs when treating each point as a class and setting the $\pmb v$ value of each sampling point to 1. In this case, the target functional of LUGSI can be rewritten as: 
\begin{align*}
R\left( f \right) &=\sum_{i=1}^l{\left( F\left( f_i \right) -Y_i \right)^{T} \left( F\left( f_i \right) -Y_i \right)}\\
                  &=\left( F\left( f \right) -Y \right)^{T} I\left( F\left( f \right) -Y \right). 
\end{align*}

Moreover, the target functional of VSVM is 
$$
R\left( f \right) =\left( F\left( f \right) -Y \right)^{T} V\left( F\left( f \right) -Y \right). 
$$
As above, the $\pmb v$ value at each point is not considered, so that $V = I$, the target functional of VSVM can be rewritten as 
\begin{align*}
R\left( f \right) &=\left( F\left( f \right) -Y \right)^{T} V\left( F\left( f \right) -Y \right)\\
                  &= \left( F\left( f \right) -Y \right)^{T} I\left( F\left( f \right) -Y \right).\qedhere
\end{align*}\end{proof}

In summary, LSSVM can be viewed as a special case of both LUGSI and VSVM when the $\pmb v$ value at each point is not considered. Therefore, LSSVM does not incorporate data structure information into its design, this may lead to its relatively poorer performance compared to LUGSI and VSVM in complex classification problems.

\begin{proposition}
    If we consider the entire dataset as a single class, then the LUGSI model degenerates into the VSVM model.
\end{proposition}
\begin{proof}
If we consider the entire dataset as a single class, then $V_i = V$, so LUGSI can be transformed to 
\begin{align*}
    R\left( f \right) &=\sum_{i=1}^m{\left( F\left( f_i \right) -Y_i \right)^{T} V_i\left( F\left( f_i \right) -Y_i \right)}\\
                      &=\left( F\left( f \right) -Y \right)^{T} V\left( F\left( f \right) -Y \right).\qedhere 
\end{align*}
\end{proof}

Therefore, it can be inferred that VSVM does not explicitly incorporate the structural information of the data into the model. Although VSVM considers the mutual positional information between the entire dataset, this positional information may not always be helpful for accurate classification. This observation can be further validated through subsequent experiments. Furthermore, LUGSI captures more accurate positional information within clusters. LUGSI constructs a corresponding statistical invariant for each granule by introducing the $\pmb v$ vector, leading to a deeper understanding of the intra-granularity structure. This may result in better classification and generalization performance compared to VSVM. We will delve into these issues in more detail in the experimental section.

\subsection{Model complexity analysis}\label{subsec2}

For the linear solution, the model solution complexity for VSVM is $O\left( 2dl^2+2d^2l \right)$, while the complexity for LUGSI is $O\left( 2mdl_i^2+md^2l_i+md^2 \right) $. However, since the number of data points in each cluster may vary, we can approximate it as the average number $\frac{l}{m}
$ of data points per cluster for comparison purposes instead of using $l_i$. At this point, the complexity of solving LUGSI can be approximated as $O\left( 2\frac{dl^2}{m}+d^2l+md^2 \right) $. When $m$ is relatively large, $$O\left( 2\frac{dl^2}{m}+d^2l+md^2 \right) \ll   O\left( 2dl^2+2d^2l \right)$$  inequalities hold. 

Therefore, the complexity of LUGSI is much lower than that of VSVM.

% Obviously, the inequality 
% $$O\left( 2mdl_i^2+md^2l_i \right) \ll   O\left( 2dl^2+2d^2l \right)$$ holds. 

\section{Experiments}\label{sec4}
To evaluate the proposed LUGSI algorithm, a series of experiments are conducted in this section. In this paper, four related SVM-type classifiers are used for experimental comparison with the
proposed LUGSI and DCTLUGSI. 
\begin{itemize}
    \item VSVM \cite{vapnik2015v}: VSVM is a model constructed based on the Fredholm integral equation and least squares loss. It considers the mutual positional relationship of the data points in its formulation.
    \item CSVM \cite{2020Deformation,chaabane2022face}: CSVM is a classic SVM with hinge loss. It aims to maximize the margin between the closest points of different categories in order to achieve effective classification.
    \item LSSVM \cite{li2016least,islam2021data}: LSSVM is the least squares version of a support vector machine that obtains a solution by solving a set of linear equations.
    \item epsVL1SVM \cite{MXZ2023}: epsVL1SVM is a model based on the Fredholm integral equation that introduces a new theory of expected risk estimation and utilizes insensitive losses. The model also considers the mutual positional relationship of the data points.
    \item DCTLUGSI (section(\ref{subsec222})): DCTLUGSI is another set of nonlinear solutions obtained by combining LUGSI with the DCT algorithm.
  \end{itemize}

  \begin{table}[]
    \caption{Datasets used in experiments. The first 14 datasets are from the UCI datasets, and the remaining datasets are from the NDC
    datasets.}\label{tab1}%
    \begin{tabular}{@{}llll|llll@{}}
    \toprule
    Item& Datasets & Featrues & Samples& Item&Datasets & Featrues & Samples \\
        \midrule
        1 & Echo & 10 & 131&11 & Titanic & 3 & 2201\\
        2 & Cleveland & 13 & 173&12 & Waveform & 21 & 5000\\
        3 & Hepatitis & 19 & 155&13 & Adult & 13 & 17887\\
        4 & Wine & 13 & 178&14 & covtype & 54 & 581012\\
        5 & Hourse & 26 & 300&15& n1000 & 32 & 1000\\
        6 & Haberman  & 3  & 306&16& n10000 & 32 & 10000\\
        7 & Ecoli & 7 & 336&17& n50000& 32 & 50000\\
        %8 & Housevotes & 16 & 435\\
        8 & Creadit & 15 & 690&18& n100000 & 32 & 100000\\
        9 & Vowel & 13 & 988&19 & n1000000 & 32 & 1000000\\
        %10 & Yeast & 8 & 1484\\
        10 & Shuttle & 9 & 1829\\
        % \bottomrule
        \botrule
        \end{tabular}
\end{table}

\subsection{Comparison of accuracy on the benchmark datasets}

To investigate the effectiveness of our LUGSI, we evaluated its performance on several datasets from the UCI repository (\url{https://archive.ics.uci.edu/datasets}), whose attributes are shown in Table \ref{tab1}. The website provides detailed descriptions for each dataset, including comprehensive information such as the dataset's origin and feature details. Due to the diversity and easy accessibility of the UCI datasets, along with their frequent citation in academic research, they are well-suited for comparative analysis with other studies. The NDC datasets are generated by David Musicants Data Generator (\url{https://research.cs.wisc.edu/dmi/svm/ndc/}), with the number of features fixed at 32. They are primarily employed to evaluate the model's performance in handling datasets of varying scales.
In this experiment, we perform 5-fold cross-validation with a grid search strategy to optimize the parameters associated with each classifier. All features are minmax scaled to the range of $[0,1].$ During the cross-validation process, we select the tradeoff parameter C from the range $\left\{ 2^{-8},2^{-7},...,2^7,2^8 \right\}.$ For small datasets ($l<800$) in LUGSI and DCTLUGSI, the range of class numbers is chosen as $\left\{1,3,7,l/2,l\right\}$, while for larger ($l\ge 800$) datasets, it is $\left\{1,3,7,l/16,l\right\}$. In the selection of the number of classes, we conduct multiple experiments by fixing other parameters and observed the variation in model performance with different numbers of classes. As for kernel setting, the radius basic function (rbf) kernel 
$$
K_{rbf}=\exp \left( -\frac{\lVert \pmb x-\pmb x' \rVert ^2}{2\delta ^2} \right) 
$$
is used and the parameter is selected from $
\left\{ 2^{-4},2^{-3},...,2^3,2^4 \right\} 
$. Regarding the selection of parameter $C$ and $\delta$, we primarily established their range based on the expertise of domain-specific experts.
In DCTLUGSI, the CRO kernel $$
K_{\gamma}\left( \pmb x,\pmb x' \right) =\varPhi ^2\left( \gamma \right) +\int_0^{\cos \left( \pmb x,\pmb x' \right)}{\frac{\exp \left( \frac{-\gamma ^2}{1+\rho} \right)}{2\pi \sqrt{1-\rho ^2}}}d\rho 
$$ is utilized, where $\gamma$ is a constant, $\varPhi \left( \gamma \right) =\int_{-\infty}^{\gamma}{\phi \left( u \right) du}$ and $\phi \left( x \right) =\frac{1}{\sqrt{2\pi}}e^{-\frac{ x^2}{2}}$. Here, a brief introduction to the concept is provided, while the detailed computation can be found in Section 5 of \cite{kafai2017croification} .

The linear results and nonlinear results are displayed in Table \ref{tab3333} and Table \ref{tab44}, respectively, with the best accuracy and shortest training time are highlighted in bold. \textquotedblleft $-$\textquotedblright{} represents that the results are not obtained if one solver takes time longer than two hours, and \textquotedblleft $//$\textquotedblright{} represents that the required memory is out of the
    capacity of our desktop. The last row displays the number of victories of our LUGSI model compared to other classifiers.

    \begin{table}[]
        \caption{Testing results on UCI datasets for linear classifiers.}\label{tab3333}%
        \begin{tabular*}{\textwidth}{@{\extracolsep\fill}llllll@{}}
        \toprule
            Dataset & VSVM   & LSSVM & CSVM & epsVL1SVM & LUGSI \\
         & Acc ($\%$) & Acc ($\%$)&Acc ($\%$)& Acc ($\%$)&Acc ($\%$)\\
         & time (s) & time (s)&time (s)& time (s)&time (s)\\
        \midrule
        Echo & $88.21\pm 2.02$ & $87.75\pm3.48$ &  $89.47 \pm0.97$& $89.89\pm1.46$ &\pmb{$90.48\pm2.78$}\\
         & 0.0034  & 0.0820 & 0.0038 & 0.0325 & 0.0019\\
         Hepatitis & $84.12\pm1.66 $ & $ 84.67\pm1.93 $ &  $82.89\pm3.14 $ & $ 83.93\pm1.57$ &\pmb{$86.31\pm1.39 $}\\
         &0.0036   &0.1290  & 0.0060 & 0.0409  & 0.0018 \\
         Wine & $99.20\pm0.27  $ & $98.98\pm0.21$    &  $98.28\pm0.16  $ &$ 99.24\pm0.21$  &\pmb{$99.55\pm0.22  $}\\
         & 0.0041   & 0.1029  &  0.0043 &  0.0492 &  0.0098 \\
        Cleveland & $94.61\pm 0.58$ & $94.36\pm0.69$ & $94.41\pm0.46$ & $94.47\pm1.22$ &\pmb{$95.51\pm0.36$}\\
         & 0.0048 & 0.1000 & 0.0038 &0.0489 & 0.0016\\
        Hourse & $80.41\pm0.34$ & $81.07\pm0.42$ & $81.30\pm0.01$ & $81.33\pm0.00$ &\pmb{$81.76\pm0.45$}\\
         & 0.0044& 0.1701 & 0.0436 &0.1368& 0.0024\\
        Haberman  & $74.20\pm0.75$ & $ 73.98\pm0.16$  & $73.46\pm0.01$ & $73.43\pm2.87$ &\pmb{$75.29\pm0.16$}\\  
         & 0.0043 & 0.0526  & 0.0052 & 0.1484 & 0.0029\\
        Ecoli & $90.43\pm0.25$ & $90.01\pm0.14$  & $92.28\pm0.20$ & $92.04\pm0.27$ &\pmb{$92.75\pm0.36$}\\
         & 0.0051 & 0.0789 & 0.0047& 0.1669 & 0.0024\\
        Creadit & $86.26\pm0.17$ & $85.83\pm0.07$& $85.50\pm0.00$ & $85.51\pm0.00$& \pmb{$86.38\pm0.00$}\\
         & 0.0126 & 0.1501 &  0.0676& 1.3464& 0.0033\\
         Vowel & $94.88\pm0.11$ & $94.75\pm0.06$  & \pmb{$97.31\pm0.11$}  & $95.76\pm0.08$ &$94.58\pm0.03$\\
  & 0.0177 & 0.1852 & 0.0232 & 2.2970 & 0.0048\\
Shuttle & $99.95\pm0.00$ &$99.95\pm0.00$ &$99.95\pm0.00$ &$ 99.95\pm0.00$ & \pmb{$99.95\pm0.00$}\\
  & 0.0548 & 0.3835 & 0.0256 &11.4773  & 0.0159\\
Titanic & $77.60\pm0.00$ & $77.60\pm0.00$ & $77.60\pm0.00$& $77.60\pm0.00$&\pmb{$77.60\pm0.00$}\\
 & 0.0662 & 0.5175  & 0.0571 & 8.1845 & 0.0183\\
Waveform & $85.60\pm0.02$ & $85.56\pm0.03$ & \pmb{$85.78\pm0.02$} &$ 85.67\pm0.02$  & $85.14\pm0.01$\\
  & 0.4715 & 3.5974 & 1.7333 & 186.9011 & 0.0857\\
  Adult & $ -$  &$ 82.06\pm0.00$  &$ 78.88\pm0.00$  & $-$  &  \pmb{$82.08\pm0.00$}\\
  & $//$  & 82.7764  & 20.3904  & $>2h $ &0.2036 \\  
        \botrule
 $\#$Best & 0  & 0 & 2& 0 & 11 \\   
        \botrule
        \end{tabular*}
        \footnotetext{The highest accuracy (Acc) value in each experimental group is indicated in bold. It is important to note the selection criteria for Acc: if the Acc values are the same, choose the group with the smaller variance. If the variances are also the same, choose the group with the shorter training time.}
        \end{table}
%---------------------

\begin{table}[]
\caption{Testing results on UCI datasets for nonlinear classifiers.}
%\centering
\label{tab44}% 
\centering
\begin{tabular}{@{}lllllll@{}}
\toprule
Dataset & VSVM   & LSSVM & CSVM & epsVL1SVM & LUGSI & DCTLUGSI
\\
& Acc ($\%$) & Acc ($\%$)&Acc ($\%$)& Acc ($\%$)&Acc ($\%$)&Acc ($\%$)
\\
& time (s) & time (s)&time (s)& time (s)&time (s)&time (s)
\\
\midrule
Echo & $89.67\pm0.80$ & $87.64\pm0.98$ & $89.72\pm1.81$ & $90.25\pm1.01$ &$90.58\pm2.12$&\pmb{$90.93\pm0.48$}
\\
& 0.0067&0.0008 & 0.0032 & 0.0332 & 0.0135&0.0123
\\
Hepatitis & $83.61\pm3.18 $ & $84.60\pm2.69 $ & $ 83.25\pm3.53 $ & $81.32\pm6.02 $ &$85.45\pm1.65 $&$\pmb{86.66\pm0.65} $
\\
&  0.0080& 0.1322 &  0.0060 & 0.0416 &  0.0178& 0.0057
\\  
Wine & $99.56\pm0.13  $ & $ 99.66\pm0.09  $ & $99.64\pm0.27   $ & $99.61\pm0.22  $ &\pmb{$99.91\pm0.09  $}&$98.49\pm0.14 $
\\
& 0.0109&  0.1062 &   0.0043 &  0.2090 &   0.1705& 0.0044
\\   
Cleveland & $94.60\pm0.28$ & $93.39\pm0.26$ & $95.07\pm0.74$ & $94.69\pm0.55$ & \pmb{$95.74\pm0.70$}&$95.09\pm0.23$
\\
 & 0.0093 & 0.1087 & 0.0037 & 0.0508 & 0.0186&0.0089
 \\
Hourse & $81.11\pm0.18$ & $79.67\pm3.48$ &$82.67\pm0.82$ & $83.39\pm0.87$ &$82.11\pm1.93$&\pmb{$ 83.67\pm0.25$}
\\
& 0.0201 & 0.1746 & 0.0262 & 0.1455 & 0.8146&0.0127
\\
Haberman  & $74.60\pm0.57$ & $74.73\pm1.03$ & $73.31\pm0.71$ & $73.25\pm0.96$ & \pmb{$75.19\pm0.59$}&$73.17\pm2.83$
\\ 
& 0.0219 & 0.0583 & 0.0094 & 0.1767& 0.0517&0.0110
\\
Ecoli & $92.90\pm0.18$ & $92.29\pm1.00$ & $92.30\pm0.82$ & $93.11\pm0.41$ &\pmb{$93.24\pm0.29$}&$93.06\pm0.08$
\\
& 0.0250 & 0.0814 &0.0095 & 0.2016& 0.0845 &0.0100
\\
Creadit & $86.36\pm0.40$ & $85.57\pm0.20$ & $86.64\pm0.06$ & $85.74\pm0.09$&  \pmb{$86.65\pm0.43$}&$ 85.81\pm0.09$
\\
& 0.1263 & 0.1777 & 0.0719 & 1.1426&   0.2885&0.0083
\\
Vowel & $99.75\pm0.03$ & \pmb{$100.00\pm0.00$}& $99.92\pm0.01$ & $99.89\pm0.01$ &$99.92\pm0.02$ & $99.02\pm0.05$
\\
& 0.4207 & 0.2229& 0.0194 & 2.4904 & 148.3893&10.3412
\\
Shuttle & $99.96\pm0.00$ & $99.95\pm0.00$ & $99.90\pm0.00$ & $99.95\pm0.00$ & \pmb{$100.00\pm0.00$}&\pmb{$100.00\pm0.00$} 
\\
& 1.9472 & 0.4958 & 0.0085 & 10.9137 & 4.5313 & 0.1417
\\
Titanic & $78.26\pm0.13$& $78.78\pm0.12$& $78.82\pm0.04$ & $78.76\pm0.05$ & $77.92\pm0.00$&\pmb{$78.92\pm0.00$}
\\
& 3.0602 & 0.6523& 0.1352 & 10.7115 & 12.1731 &0.1319
\\
Waveform &$ \pmb{90.15\pm0.01}$ & $90.00\pm0.02$& $90.00\pm0.01$ & $89.94\pm0.02$ &$ 89.05\pm0.03$ &$86.75\pm0.02$
\\
&31.6183  & 4.2664  & 1.5425 & 208.0701 & 106.2753 &0.2314
\\
Adult &$-$ & $75.01\pm0.00$&$81.52\pm0.00 $  & $-$ & $ 76.52\pm0.00$ & \pmb{$82.58\pm0.01$}
\\
&$//$  & 97.1948  &28.4001  &$>2h $  & 2443.0999 & 36.5046
\\
\botrule
$\#$Best & 1  & 1 & 0& 0 & 6 & 6 
\\ 
\botrule
\end{tabular}
%\footnotetext{Source: This is an example of table footnote. This is an example of table footnote.}
\end{table}

From these results, we can draw the following interesting observations.

    (1) The results of UCI datasets in the linear case: It can be clearly observed that out of the 13 datasets, LUGSI demonstrates superior performance on 11 datasets. Particularly, the classification performance of VSVM surpasses that of LSSVM overall, which might be attributed to the utilization of structural information by VSVM. This signifies the importance of incorporating structural information into classification tasks. While LUGSI also considers the structural aspect, it differs from VSVM by employing a clustering approach to process the data, emphasizing the influence of local information on classification. Moreover, the training time of the LUGSI model is notably superior to other models, especially when compared to LSSVM and epsVL1SVM. It is worth noting that due to the unique nature of the \textquotedblleft Shuttle\textquotedblright{} and \textquotedblleft Titanic\textquotedblright{} datasets in linear space, the accuracy fluctuations among various models are not significant for these datasets. However, the use of the LUGSI algorithm still substantially
    accelerates the training speed. In summary, compared to other methods, LUGSI significantly reduces training time and consistently maintains the highest accuracy on the majority of experimental datasets.

    (2) The results of UCI datasets in the nonlinear case: Evidently, the nonlinear LUGSI classifier exhibits superior classification performance on the majority of datasets when compared to other classifiers. It's noteworthy that in most datasets, LUGSI outperforms VSVM in classification performance, further validating the advantage of LUGSI's utilization of local data information. Moreover, DCTLUGSI effectively harnesses the distinctive attributes of the linear LUGSI model, utilizing the DCT algorithm to reduce computation time for the LUGSI model in the nonlinear space. This achievement empowers LUGSI to attain exceptional classification performance in the nonlinear space while concurrently reducing training time.

\subsection{The impact of clustering and on accuracy and training time}

To present the impact of the number of clusters on training time and accuracy for UCI and NDC datasets, and to further validate the necessity of considering the number of clusters as a parameter, we plot variation of training time and  accuracy along with number of clusters. Fig.\ref{Fig:1}, Fig.\ref{Fig:2} and  Fig.\ref{Fig:3} illustrate the results on the datasets. In this experiment, we keep the other parameters fixed and treated the number of clusters $m$, as a variable. For small datasets, we set the step size of $m$ based on the dataset size to be 10, 20, 50, and 100, respectively. For the two extremely large datasets, \textquotedblleft covtype\textquotedblright{} and \textquotedblleft n1000000\textquotedblright, we select 8 values of $m$ from the range 1 to 1000 for the former, and 10 values of $m$ from the same range for the latter, to conduct the experiments.
\begin{figure*}[]
    \centering
    \subfigure[Hepatitis]{\includegraphics[height=1.24 in]{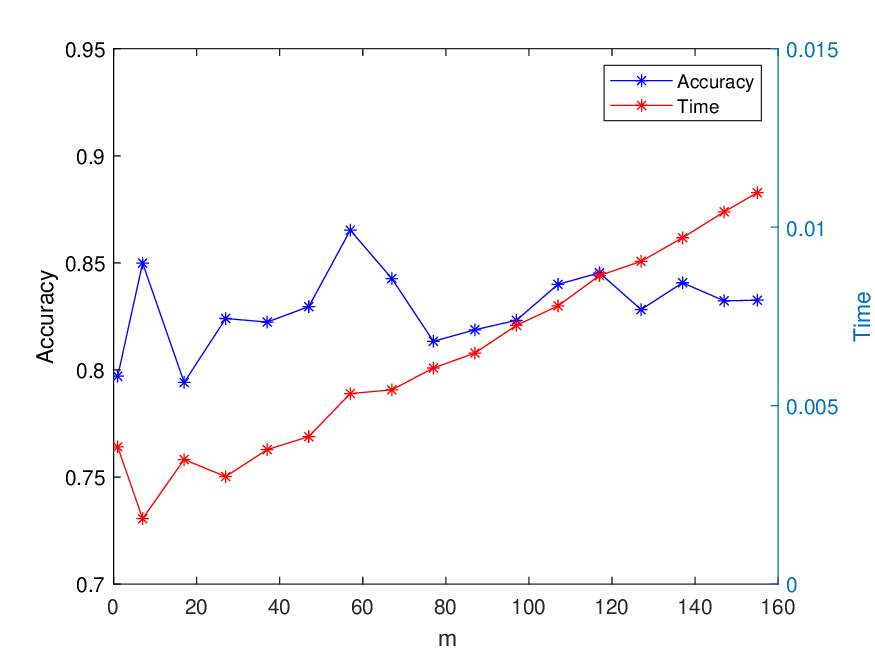}}
    \hspace{0 pt}
    \subfigure[Cleveland]{\includegraphics[height=1.24 in]{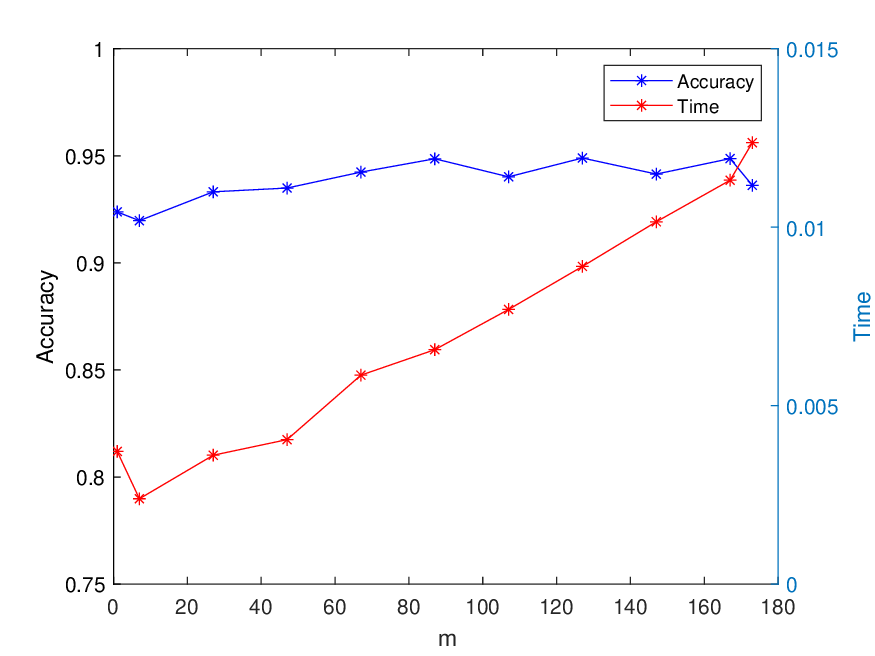}}
    \hspace{0 pt}
    \subfigure[Creadit]{\includegraphics[height=1.24 in]{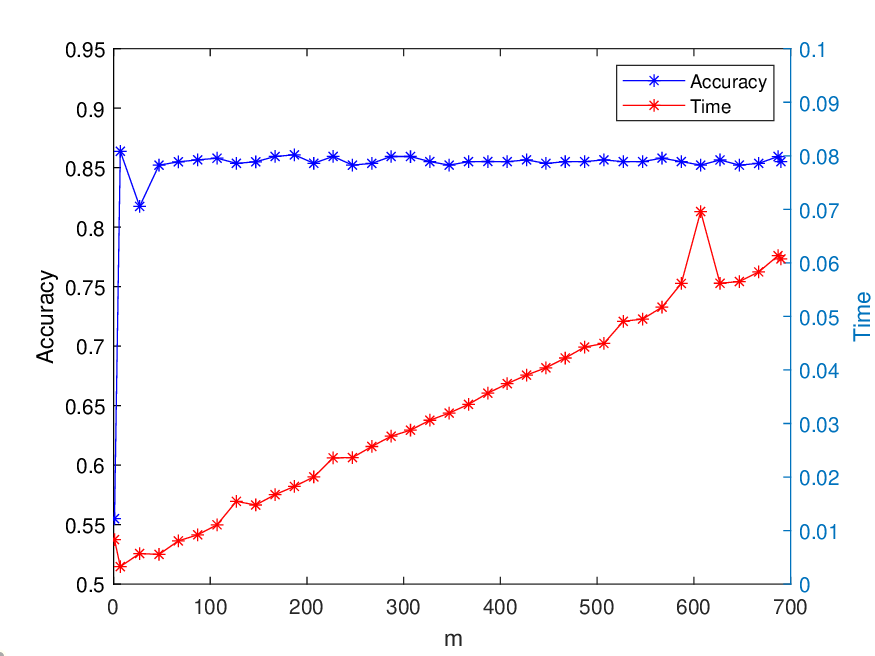}}
    \\
    \subfigure[Vowel]{\includegraphics[height=1.24 in]{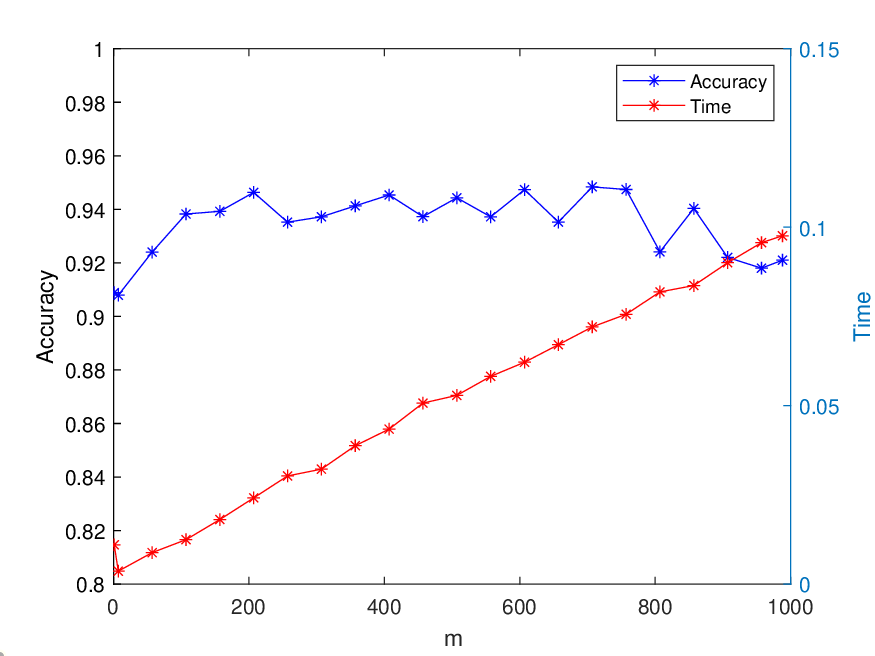}}
    \hspace{0 pt}
    \subfigure[Shuttle]{\includegraphics[height=1.24 in]{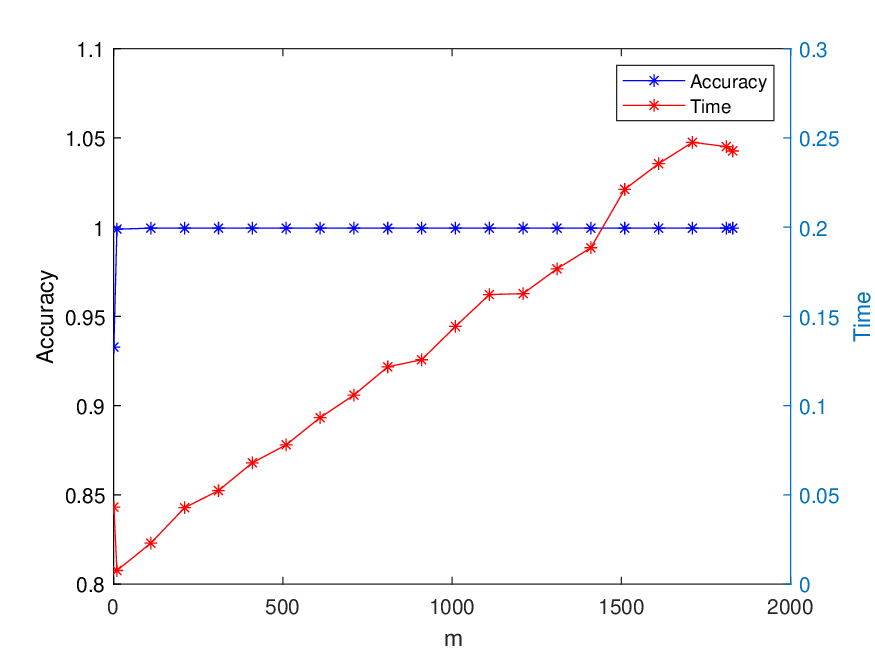}}
    \hspace{0 pt}
    \subfigure[Titanic]{\includegraphics[height=1.24 in]{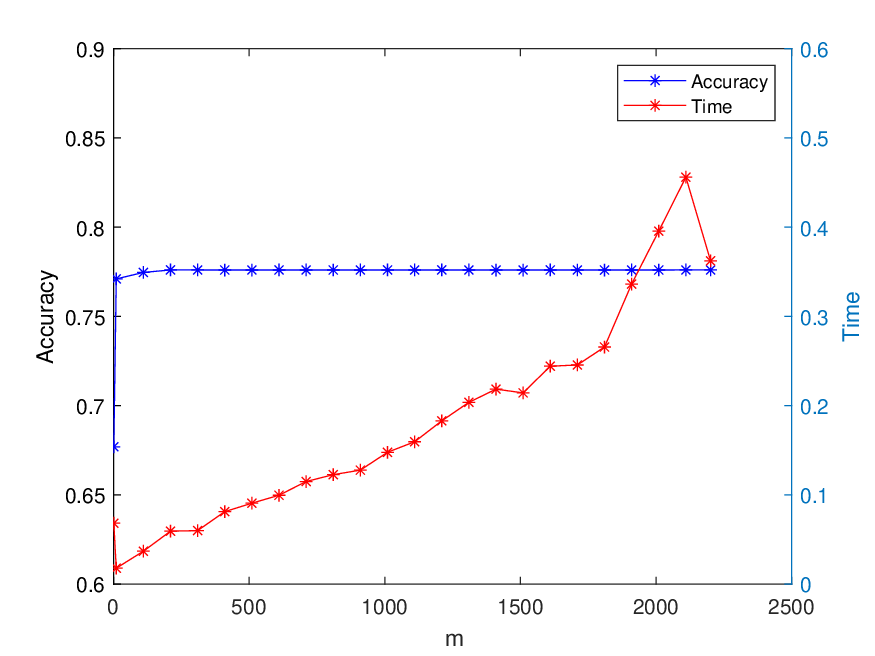}}
    \caption{The impact of clustering in linear space on accuracy and training time on small datasets from UCI.}\label{Fig:1}
\end{figure*}
%%%%%%%%%%%%%%%%%%%%%%%%%%%%%%%%%%%%%

\begin{figure*}[]
    \centering
    \subfigure[Hepatitis]{\includegraphics[height=1.24 in]{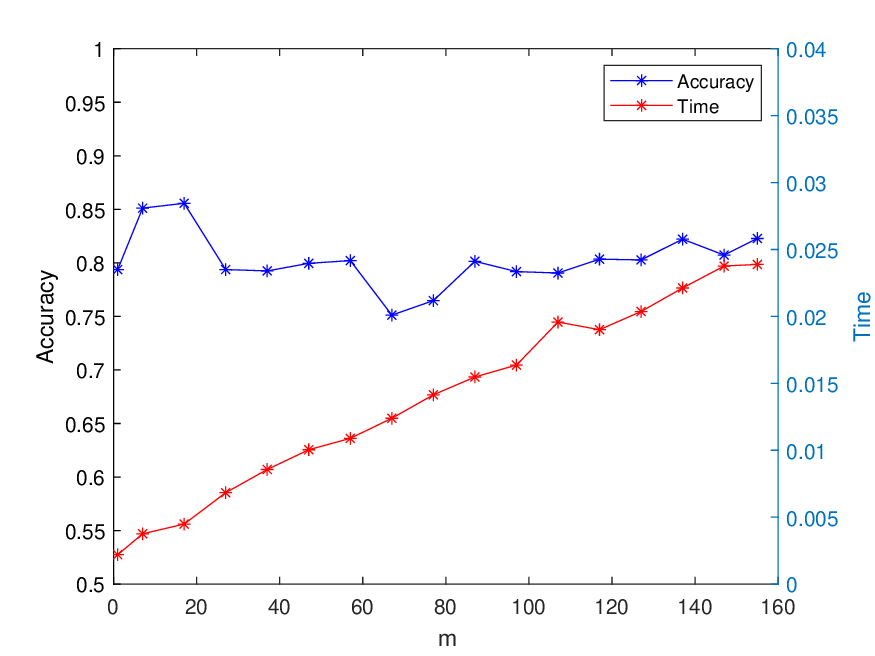}}
    \hspace{0 pt}
    \subfigure[Cleveland]{\includegraphics[height=1.24 in]{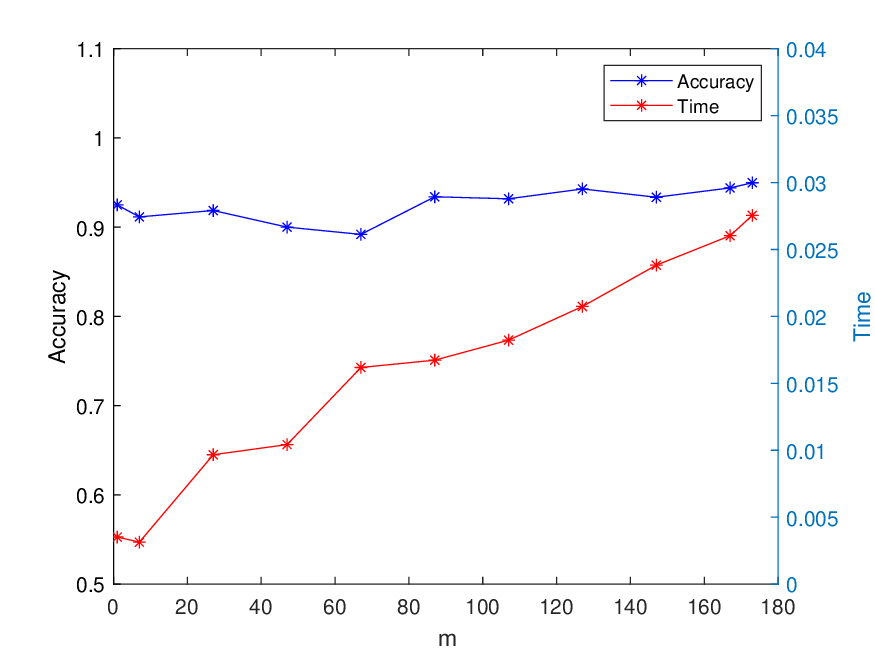}}
    \hspace{0 pt}
    \subfigure[Creadit]{\includegraphics[height=1.24 in]{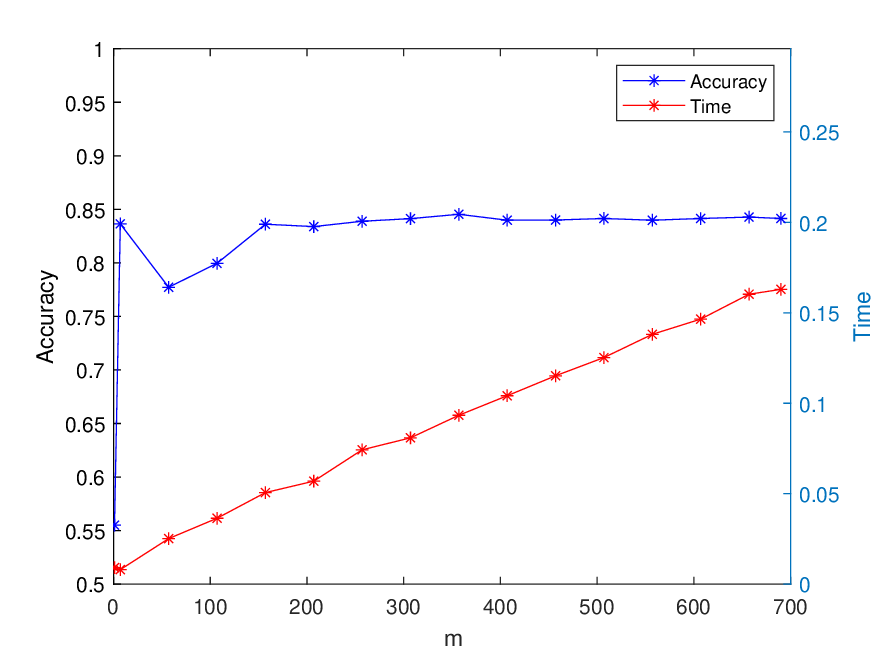}}
    \\
    \subfigure[Vowel]{\includegraphics[height=1.24 in]{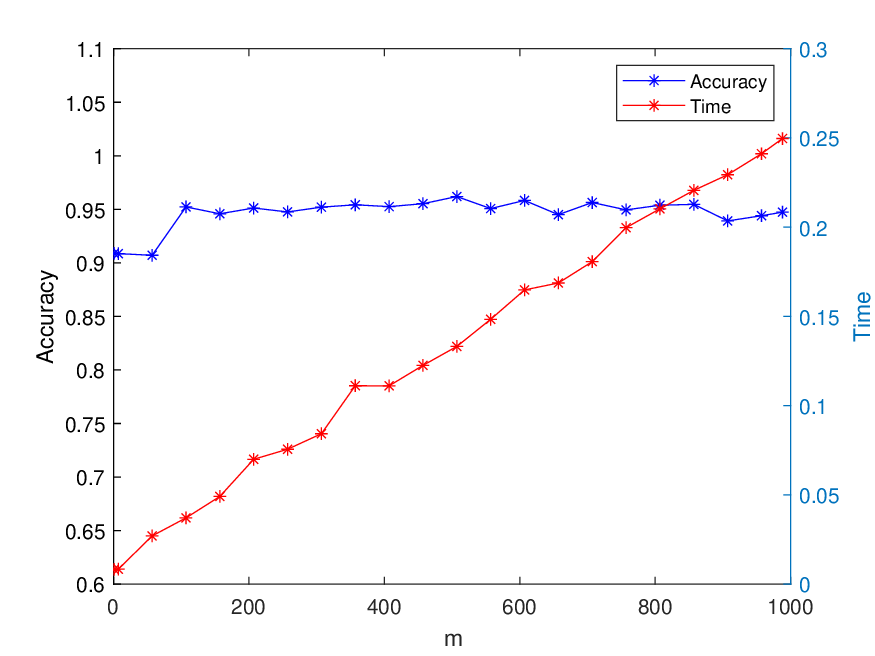}}
    \hspace{0 pt}
    \subfigure[Shuttle]{\includegraphics[height=1.24 in]{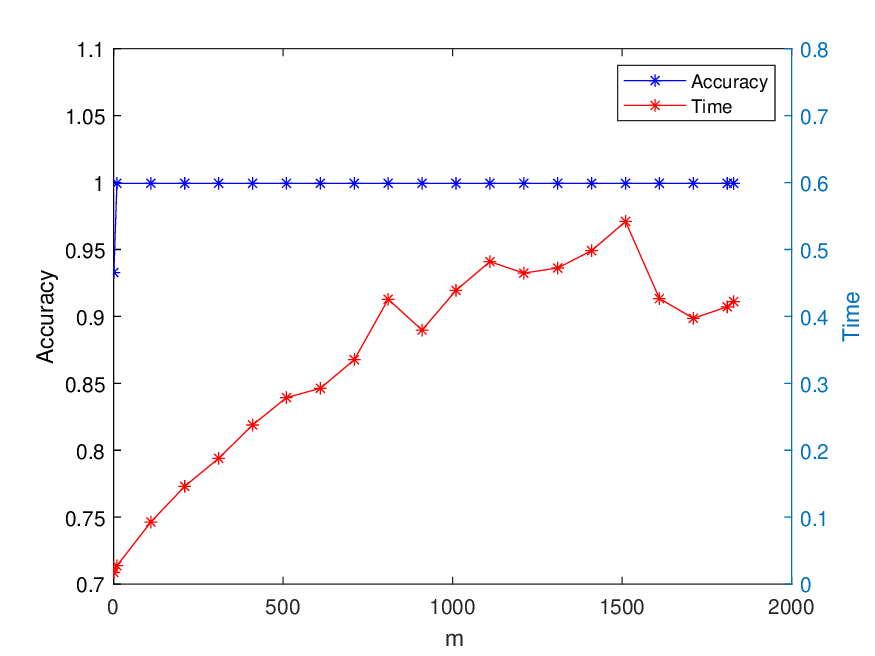}}
    \hspace{0 pt}
    \subfigure[Titanic]{\includegraphics[height=1.24 in]{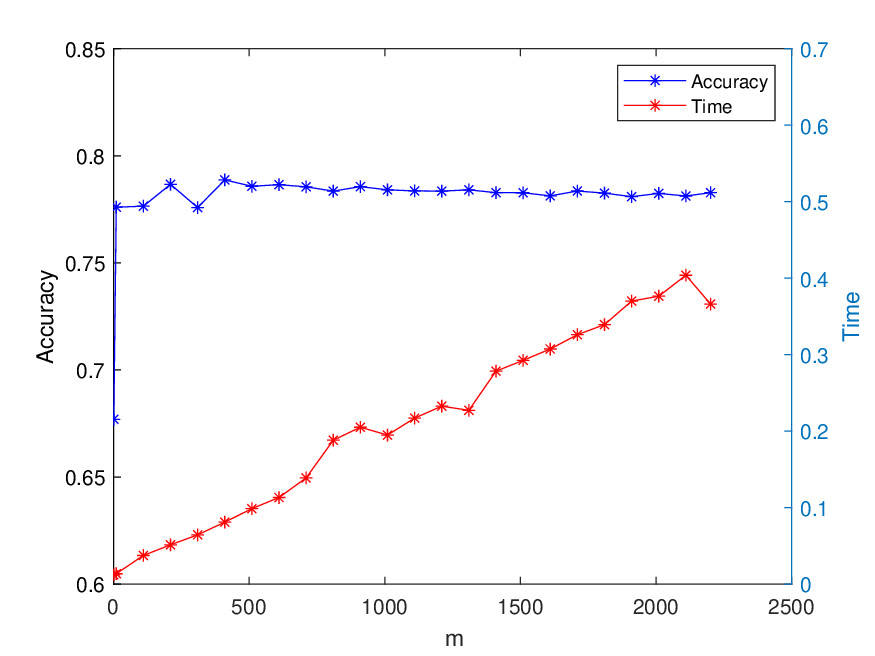}}
    \caption{The impact of clustering in nonlinear space on accuracy and training time on small datasets from UCI.}\label{Fig:2}
\end{figure*}

%%%%%%%%%%%%%%%%%%%%%%%%%%%%%%%%%%%%%

%%%%%%%%%%%%%%%%%%%%%%%%%%%%%%%%%%%%%
    \begin{figure}[htbp]
        \centering
        \subfigure[covtype.]{
        \begin{minipage}[t]{0.4\linewidth}
        \centering
        \includegraphics[width=2in]{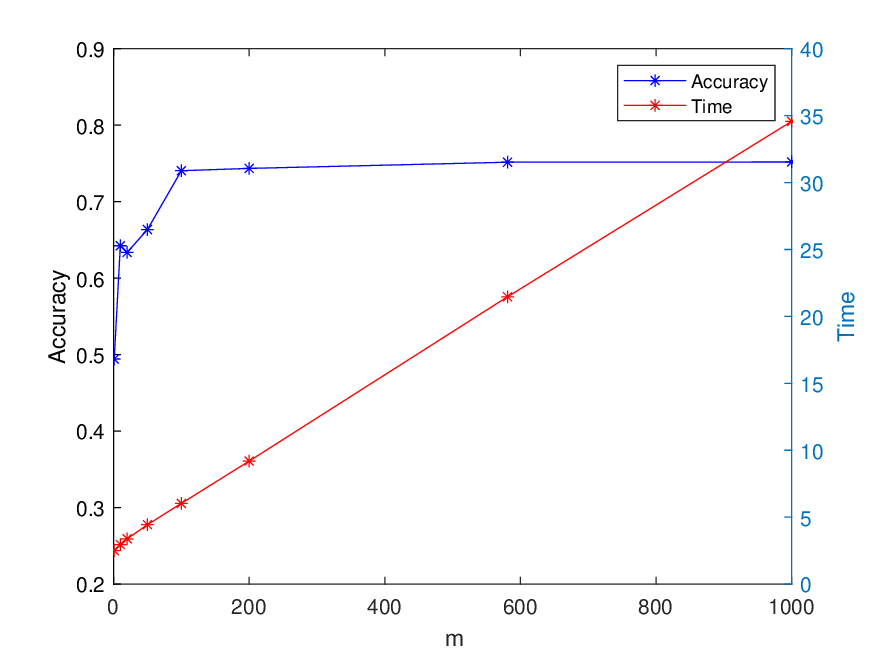}
        %\caption{fig1}
        \end{minipage}%
        }%
        \subfigure[n1000000.]{
        \begin{minipage}[t]{0.4\linewidth}
        \centering
        \includegraphics[width=2in]{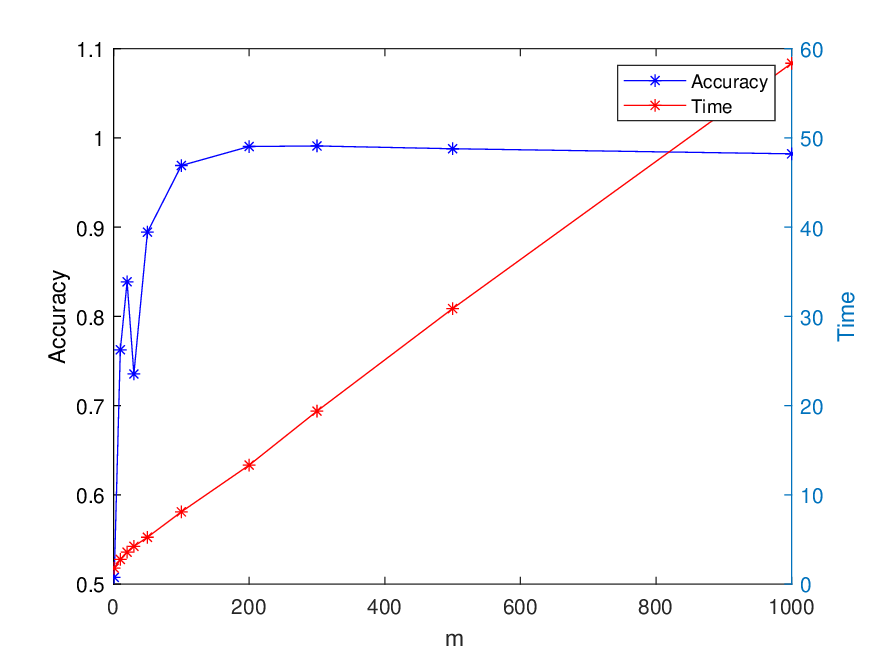}
        %\caption{fig2}
        \end{minipage}%
        }%  
        
        \centering
    \caption{ The impact of clustering on accuracy and training time of large-scale datasets.}\label{Fig:3}
    \end{figure}
%%%%%%%%%%%%%%%%%%%%%%%%%%%%%%%%%%%%%

(1) The impact of clustering and on accuracy and training time on small datasets form UCI: Fig.\ref{Fig:1} and Fig.\ref{Fig:2} clearly demonstrate that as the number of sample clusters increases, the training time of the classifier approximately grows linearly. Concurrently, the accuracy of the classifier fluctuates, showing both improvements and declines. These results demonstrate the importance of selecting an appropriate number of clusters for the classifier, as it can significantly enhance accuracy of LUGSI model within a certain time constraint.
    
(2) The impact of clustering and on accuracy and training time of large-scale datasets: The results presented in Fig.\ref{Fig:3} are consistent with those obtained from the small sample set. These experimental findings demonstrate the feasibility of the LUGSI model in handling large-scale datasets.

\subsection{The influence of global and local information on the accuracy and training time of NDC datasets with different scales}
In this section, we aim to validate the impact of global and local information on model training accuracy and training time. To achieve this, we apply LUGSI to different-scale NDC datasets (1000, 10000, 50000, 100000, 1000000) for experimentation. In the experiments, if no clustering is applied to the data, it signifies that LUGSI is utilizing the global information of the data; conversely, if data clustering is performed, it indicates that LUGSI is utilizing the local information of the data. The experimental results, as depicted in Fig.\ref{Fig:4}, reveal that the blue bars represent accuracy and training time obtained using global information, while the red bars represent accuracy and training time obtained using local information. 

\begin{figure}[htbp]
    \centering
    \subfigure[Accuracy.]{
    \begin{minipage}[t]{0.4\linewidth}
    \centering
    \includegraphics[width=2in]{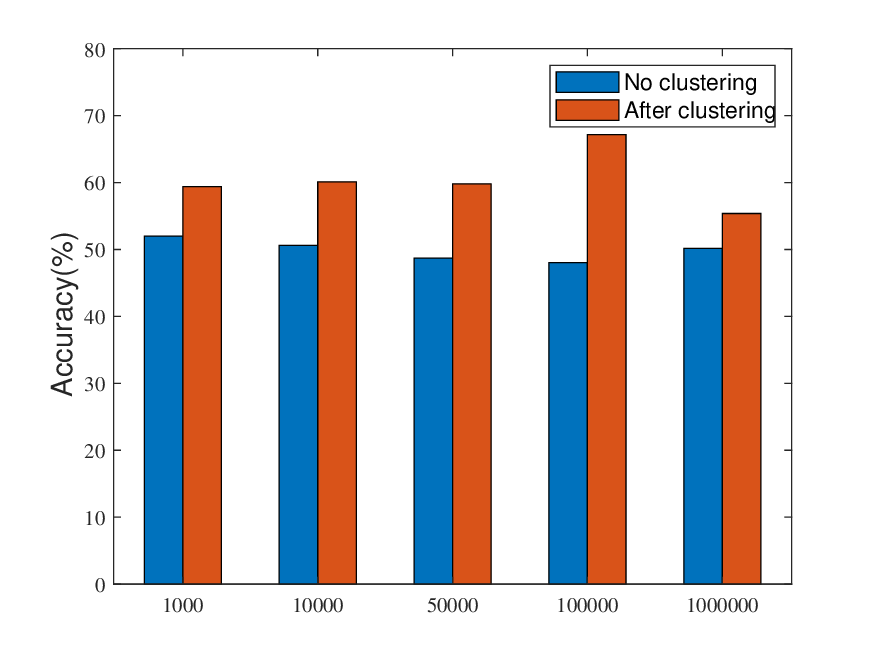}
    %\caption{fig1}
    \end{minipage}%
    }% 
    \subfigure[Training time.]{
    \begin{minipage}[t]{0.4\linewidth}
    \centering
    \includegraphics[width=2in]{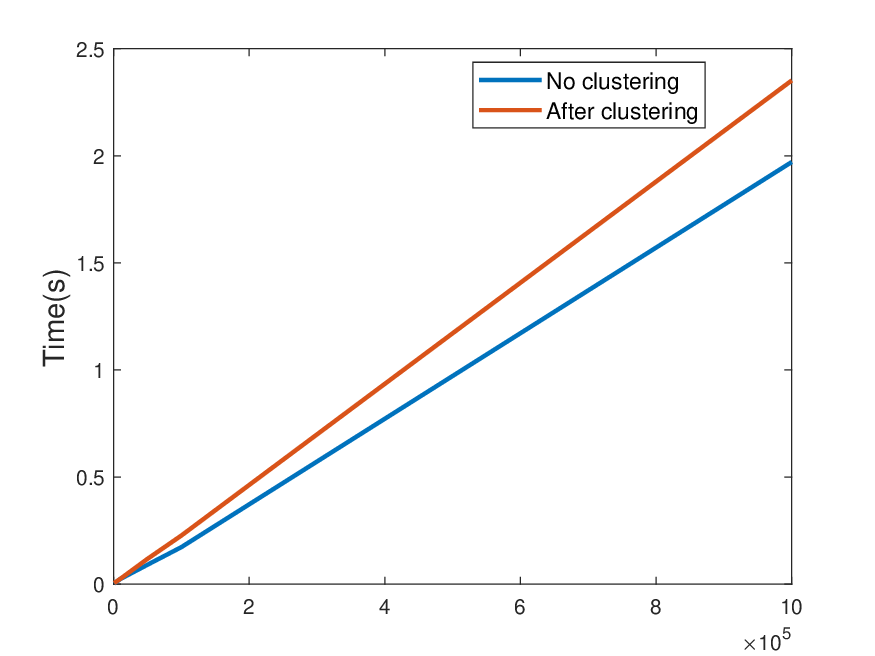}
    %\caption{fig1}
    \end{minipage}%
    }% 
    \centering
    \caption{The effect of clustering on the accuracy and training time on NDC datasets of different sizes.}\label{Fig:4}
    \end{figure}

      Clearly, utilizing local information significantly improves the model's accuracy on datasets of different scales. At the same time, it also demonstrates the feasibility of LUGSI in handling classification problems with large-scale datasets. Therefore, Fig.\ref{Fig:4} further confirms the previous conclusion that using local information on the sample set not only enables the LUGSI model to effectively handle large-scale datasets but also does not significantly increase the required training time.

    \section{Conclusion}\label{sec4}

    This paper introduces a statistical invariant algorithm based on granularity. This algorithm utilizes the $K$-means clustering method to construct granules and introduces $\pmb v$ vectors to construct statistical invariants for each granule. By maximizing the distance between classes, a large invariant matrix is converted into multiple smaller invariant matrices, reducing the complexity of the invariant matrix and enabling the LUGSI model to effectively address classification problems on large-scale datasets with shorter training times. The experimental results demonstrate that LUGSI outperforms LSSVM and VSVM in terms of generalizability and training times, indicating that incorporating finer structural positional information into the model through granular construction is beneficial for the classification task. 
    
    This paper utilizes the $K$-means clustering method to construct granules. Finding a more suitable clustering method or determining the optimal number of clusters is interesting. Additionally, the risk measure employed in this paradigm adopts the least squares loss. Exploring potentially more suitable alternative loss metrics will be a focus of our future research.

    \bmhead{Acknowledgments}
    
    This work is supported in part by National Natural Science Foundation of China (Nos. 12271131, 62106112, 62066012), and in part by the Key Laboratory of Engineering Modeling and Statistical Computation of Hainan Province.

%%=============================================%%
%% For presentation purpose, we have included  %%
%% \bigskip command. please ignore this.       %%
%%=============================================%%

%%=============================================%%
%% For presentation purpose, we have included  %%
%% \bigskip command. please ignore this.       %%
%%=============================================%%

%%=============================================%%
%% For presentation purpose, we have included  %%
%% \bigskip command. please ignore this.       %%
%%=============================================%%

%%=============================================%%
%% For presentation purpose, we have included  %%
%% \bigskip command. please ignore this.       %%
%%=============================================%%

%%=============================================%%
%% For presentation purpose, we have included  %%
%% \bigskip command. please ignore this.       %%
%%=============================================%%

\section*{Declarations}
\begin{itemize}
    \item Competing interests: The authors declare that they have no known competing financial interests or personal relationships that could have appeared to influence the work reported in the paper \textquotedblleft Learning using granularity statistical invariants for classification\textquotedblright.
    \item Ethical and informed consent for data used: This paper exclusively utilized datasets that have been authorized or made publicly available.
    \item Data availability and access: The datasets analysed during the current study are available in the UCI and NDC repositories, \url{https://archive.ics.uci.edu/datasets}, \url{https://research.cs.wisc.edu/dmi/svm/ndc/}.
\end{itemize}





%% file, anSd delete the associated \verb+\bibliography+ commands.                            %%
%%===========================================================================================%%

\bibliography{sn-bibliography}% common bib file
%%\bibliographystyle{sn-basic}

\newpage

\section*{Authors contribution statement}

Conceptualization: Ting-Ting Zhu; Writing-original draft preparation: Ting-Ting Zhu; Writing-review and editing: Yuan-Hai Shao, Chun-Na Li, Tian Liu; Funding acquisition: Yuan-Hai Shao, Chun-Na Li; Supervision: Yuan-Hai Shao, Chun-Na Li.

\let\cleardoublepage\clearpage
\end{document}